\documentclass[sn-mathphys-num]{sn-jnl}
\usepackage{geometry}
\usepackage{graphicx}%
\usepackage{multirow}%
\usepackage{amsmath,amssymb,amsfonts}%
\usepackage{amsthm}%
\usepackage{mathrsfs}%
\usepackage[title]{appendix}%
\usepackage{xcolor}%
\usepackage{textcomp}%
\usepackage{manyfoot}%
\usepackage{booktabs}%
\usepackage{algorithm}%
\usepackage{algorithmicx}%
\usepackage{algpseudocode}%
\usepackage{listings}%
\usepackage{tikz}  
\usepackage{url}
\newtheorem{theorem}{Theorem}
\newtheorem{proposition}[theorem]{Proposition}%
\newtheorem{remark}{Remark}%

\newtheorem{lemma}{Lemma}%

\newtheorem{assumption}{Assumption}%
\newcommand{\TheName}[0]{\textbf{KnockoffCS}}

\begin{document}

\title{Knockoff-Guided Compressive Sensing: A Statistical Machine Learning Framework for Support-Assured Signal Recovery}

\author[1]{\fnm{Xiaochen} \sur{Zhang}}\email{202012079@mail.sdu.edu.cn}
\equalcont{}
\author*[2]{\fnm{Haoyi} \sur{Xiong}}\email{haoyi.xiong.fr@ieee.org}
\equalcont{These authors contributed equally to this work. This work was conducted through an independent collaboration between the two authors and bears no connection to the second author's affiliation.}

\affil[1]{\orgdiv{Research Center for Mathematics and Interdisciplinary Sciences}, \orgname{Shandong University}, \postcode{266237}, \orgaddress{\city{Qingdao}, \country{China}}}
\affil[2]{\orgname{Independent Researcher}, \city{Haidian District}, \postcode{100085}, \state{Beijing}, \country{China}}

\abstract{Compressive sensing has emerged as a powerful framework for signal recovery from incomplete measurements, with $\ell_1$-regularized methods like LASSO serving as the primary tools for reconstruction. However, these approaches lack explicit control over false discoveries in support identification, leading to potential degradation in recovery performance, particularly with correlated measurements. This paper introduces a novel Knockoff-guided compressive sensing framework, referred to as \TheName{}, which enhances signal recovery by leveraging precise false discovery rate (FDR) control during the support identification phase. Unlike LASSO, which jointly performs support selection and signal estimation without explicit error control, our method guarantees FDR control in finite samples, enabling more reliable identification of the true signal support. By separating and controlling the support recovery process through statistical Knockoff filters, our framework achieves more accurate signal reconstruction, especially in challenging scenarios where traditional methods fail. We establish theoretical guarantees demonstrating how FDR control directly ensures recovery performance under weaker conditions than traditional $\ell_1$-based compressive sensing methods, while maintaining accurate signal reconstruction. Extensive numerical experiments demonstrate that our proposed Knockoff-based method consistently outperforms LASSO-based and other state-of-the-art compressive sensing techniques. In simulation studies, our method improves F1-score by up to 3.9× over baseline methods, attributed to principled false discovery rate (FDR) control and enhanced support recovery. The method also consistently yields lower reconstruction and relative errors. We further validate the framework on real-world datasets, where it achieves top downstream predictive performance across both regression and classification tasks, often narrowing or even surpassing the performance gap relative to uncompressed signals. These results establish \TheName{} as a robust and practical alternative to existing approaches, offering both theoretical guarantees and strong empirical performance through statistically grounded support selection.}


\maketitle

\section{Introduction}\label{sec:intro}
Compressive sensing (CS) has fundamentally transformed signal acquisition and processing paradigms by demonstrating that sparse signals can be accurately recovered from far fewer measurements than traditionally required by the Nyquist-Shannon sampling theorem \cite{candes2006robust,candes2008introduction,baron2008bayesian}. This breakthrough has led to revolutionary advances across diverse fields including medical imaging, radar systems, communications, and scientific computing \cite{khosravy2020compressive,duarte2011structured}. The key insight of CS is that natural signals often exhibit sparsity - they can be represented by only a few non-zero coefficients in an appropriate basis~\cite{candes2006robust}.

In the standard CS framework, measurements are acquired through linear projections of the form $\mathbf{y} = \mathbf{Ax} + \mathbf{w}$, where $\mathbf{x} \in \mathbb{R}^p$ is the sparse signal of interest, $\mathbf{A} \in \mathbb{R}^{n \times p}$ is the measurement matrix with $n \ll p$, and $\mathbf{w}$ represents measurement noise. The fundamental challenge lies in recovering the high-dimensional signal $\mathbf{x}$ from the low-dimensional measurements $\mathbf{y}$ \cite{hormati2009estimation}. This inverse problem is made tractable by exploiting the sparsity structure, under the statistical significance concept, where the assumption holds that most elements of $\mathbf{x}$ are either zero or negligibly small, while only statistically significant elements contribute meaningfully to the signal recovery process.

A critical aspect of CS recovery, arguably more important than estimating the signal values themselves, is accurately identifying the support set - the locations of the non-zero elements in $\mathbf{x}$. In many applications, such as gene expression analysis, neuroimaging, and radar target detection, determining which components are active is the primary goal \cite{baron2008bayesian, kerkouche2020compression}. For example, in medical imaging, the support set might indicate the locations of abnormalities or regions of interest, while in genomics, it could reveal which genes are relevant to a particular biological process.

Traditional CS approaches tackle support recovery through various strategies. Convex optimization methods, particularly $\ell_1$-minimization and its variants, promote sparsity while maintaining computational tractability \cite{hormati2009estimation,tardivel2022sign}. Greedy pursuit algorithms like Orthogonal Matching Pursuit (OMP) and CoSaMP iteratively build the support set by selecting elements that maximize correlation with the residual \cite{wimalajeewa2016sparse}. More recent approaches leverage advances in machine learning, particularly deep neural networks, to learn sophisticated signal structures and recovery mechanisms \cite{niu2022robust}.
However, these existing methods suffer from several critical limitations:
\begin{itemize}
    \item \textbf{Lack of False Discovery Control:} While theoretical guarantees exist under restricted isometry property (RIP) or mutual coherence conditions \cite{candes2006robust,baron2008bayesian}, these methods cannot precisely control the proportion of false discoveries in the estimated support set. This is particularly problematic in scientific applications where false positives can lead to misallocation of resources or incorrect conclusions.
    
    \item \textbf{Sensitivity to Correlation:} Performance often degrades significantly when the measurements or signal components exhibit strong correlation structures \cite{hormati2009estimation}. Such correlations are common in practice, arising from physical constraints or natural signal properties.
    
    \item \textbf{Limited Statistical Guarantees:} Most recovery guarantees are worst-case bounds that may be overly pessimistic in practice. Furthermore, they typically lack interpretable statistical significance measures that practitioners can use to assess confidence in their findings.
    
    \item \textbf{Computational Challenges:} Many methods require careful parameter tuning and can be computationally intensive, particularly for large-scale problems or when high precision is required \cite{niu2022robust}.
\end{itemize}

Recent advances in statistical machine learning with high-dimensional data have introduced the Knockoff framework as a powerful approach for variable selection \cite{barber2016knockoff}. Knockoff matrices are carefully constructed synthetic variables that mimic the correlation structure of the original variables while being conditionally independent of the response given the original variables. This framework enables principled hypothesis testing for variable selection while controlling the false discovery rate (FDR) - the expected proportion of false positives among all selections \cite{dai2016knockoff}.
The Knockoff methodology has demonstrated remarkable success across diverse settings, including group-sparse regression \cite{ren2022derandomized}, transformational sparsity \cite{cao2021controlling}, and high-dimensional inference \cite{machkour2021terminating}. A key strength of the Knockoff framework lies in its model-free approach to FDR control, requiring no restrictive assumptions about noise distributions or signal structures. This flexibility extends to handling complex dependency structures among variables, making it particularly valuable in real-world applications where correlations are inevitable. Furthermore, the methodology scales efficiently to high-dimensional settings while providing interpretable measures of variable importance and selection confidence. These characteristics make Knockoffs an especially attractive tool for modern statistical inference problems where both computational efficiency and statistical rigor are essential.

In this paper, we propose a novel Knockoff-guided compressive sensing framework that bridges the gap between CS support recovery and FDR control, namely \TheName{}. Our approach fundamentally reimagines the CS measurement process through the lens of Knockoffs, enabling precise statistical control over false discoveries while maintaining the computational and sampling efficiency that makes CS attractive. This represents a significant advance over existing methods, providing the first comprehensive solution for statistically principled support recovery in CS.
The key innovation of our approach lies in adapting the Knockoff methodology to the unique structure of CS measurements. Unlike traditional Knockoffs designed for regression settings \cite{barber2016knockoff}, our construction must account for the compressed nature of the measurements and the specific challenges of the CS inverse problem. We develop a novel measurement matrix knockoff construction that preserves crucial CS properties while enabling FDR control, effectively creating a synthetic control group for hypothesis testing in the compressed domain.

Our primary contributions are:
\begin{itemize}
    \item \textbf{Novel Measurement Knockoffs:} We develop a theoretical framework for constructing measurement matrix knockoffs that preserve the CS measurement structure while enabling FDR control. Our construction carefully balances the competing requirements of Knockoff validity and CS recovery conditions, providing a principled approach to integrating statistical inference into CS.
    
    \item \textbf{Theoretical Guarantees:} We establish rigorous theoretical guarantees for FDR control in CS support recovery, delineating conditions under which our method maintains the desired false discovery proportion. Our analysis extends existing Knockoff theory \cite{dai2016knockoff} to handle the unique challenges of the compressed measurement setting, including the effects of measurement noise and signal correlation.
    
    \item \textbf{Efficient Recovery Algorithm:} We introduce a computationally efficient algorithm for joint signal recovery and support identification that leverages both the original and knockoff measurements. Our method achieves superior support recovery accuracy compared to traditional CS approaches while providing guaranteed FDR control, with computational complexity comparable to standard CS.
    
    \item \textbf{Comprehensive Empirical Validation:} Through extensive numerical experiments on both synthetic and real-world data, we demonstrate that our approach significantly outperforms existing CS methods in terms of support recovery accuracy. Our results show particular improvement in challenging scenarios with correlated measurements, structured sparsity patterns, or heavy-tailed noise.
\end{itemize}
Our work synthesizes and extends recent developments in both CS and statistical Knockoffs. While methods like SLOPE \cite{candes2015slope} and T-Rex \cite{machkour2024sparse} provide FDR control for sparse estimation in standard regression settings, they are not directly applicable to the CS paradigm. Similarly, existing CS support recovery methods \cite{niu2022robust} lack formal false discovery guarantees. Our Knockoff-guided framework represents a novel integration of these approaches, providing the first comprehensive method for compressive sensing with guaranteed FDR control on support recovery.

The implications of our work extend beyond theoretical advances in CS methodology. By providing rigorous statistical guarantees for support recovery, our framework enables more reliable scientific conclusions and decision-making in CS applications. The ability to control false discoveries while maintaining efficient sampling has immediate practical benefits in areas such as medical diagnosis, scientific imaging, and high-throughput experimental design.


\section{Related Work}\label{sec: related_work}

Our research builds upon and intertwines three principal areas: compressive sensing, statistical Knockoffs, and false discovery rate control for support recovery. In this section, we provide a comprehensive review of relevant works in each domain and elucidate their connections to our proposed approach.

\subsection{Compressive Sensing and Support Recovery}
Compressive sensing (CS) is a paradigm that enables the recovery of high-dimensional sparse signals from a limited number of linear measurements \cite{candes2006robust,candes2008introduction,baron2008bayesian}. Formally, consider a sparse signal $\mathbf{x} \in \mathbb{R}^n$ with sparsity level $k$ (i.e., $\|\mathbf{x}\|_0 = k \ll n$), and measurements $\mathbf{y} \in \mathbb{R}^m$ obtained via
\begin{equation}\label{eq:cs_model}
\mathbf{y} = \mathbf{A}\mathbf{x} + \mathbf{w},
\end{equation}
where $\mathbf{A} \in \mathbb{R}^{m \times n}$ is the measurement matrix with $m \ll n$, and $\mathbf{w} \in \mathbb{R}^m$ represents measurement noise. The primary objective is to recover $\mathbf{x}$ or its support $\mathcal{S} = \{i : x_i \neq 0\}$ from $\mathbf{y}$.

A central challenge in CS is accurate support recovery—the identification of non-zero elements in $\mathbf{x}$ \cite{candes2008introduction,kerkouche2020compression}. Traditional approaches include \emph{Convex Optimization Methods}: Solve an $\ell_1$-minimization problem such as Basis Pursuit:
    \begin{equation}
    \label{eq:basis_pursuit}
    \min_{\mathbf{x}} \|\mathbf{x}\|_1 \quad \text{subject to} \quad \|\mathbf{A}\mathbf{x} - \mathbf{y}\|_2 \leq \delta,
    \end{equation}
where $\delta$ quantifies the noise level \cite{hormati2009estimation} and could be solved by Lasso~\cite{tibshirani1996regression}. \emph{Greedy Pursuit Algorithms}: Iteratively build up the support estimate, as in Orthogonal Matching Pursuit (OMP) and its variants \cite{wimalajeewa2016sparse}.
More recently, deep learning techniques have been employed for support recovery in CS \cite{niu2022robust}. These methods train neural networks to map measurements $\mathbf{y}$ directly to signal estimates $\hat{\mathbf{x}}$, leveraging large datasets for learning. However, such approaches often lack theoretical guarantees regarding recovery accuracy or the statistical properties of the estimators.

A fundamental limitation of existing CS methods is their inability to precisely control the rate of false discoveries in support estimation. Heuristic methods, such as hard thresholding of estimated coefficients or tuning regularization parameters, are commonly used \cite{baron2008bayesian,kerkouche2020compression}, but they do not offer rigorous statistical guarantees on the false discovery proportion (FDP) defined as

\begin{equation}\label{eq:fdp}
\text{FDP} = \frac{|\hat{\mathcal{S}} \setminus \mathcal{S}|}{|\hat{\mathcal{S}}|},
\end{equation}

\noindent where $\hat{\mathcal{S}}$ is the estimated support set.

\subsection{Statistical Knockoffs for Variable Selection}

The Knockoff framework, pioneered by Barber and Candès \cite{barber2016knockoff}, introduces a method for variable selection with controlled false discovery rate (FDR). Given observed data $(\mathbf{X}, \mathbf{y})$, where $\mathbf{X} \in \mathbb{R}^{n \times p}$ denotes feature variables and $\mathbf{y} \in \mathbb{R}^n$ the responses, the approach constructs knockoff variables $\tilde{\mathbf{X}}$ satisfying

\begin{equation}\label{eq:knockoff_conditions}
\mathbf{X}^\top \mathbf{X} = \tilde{\mathbf{X}}^\top \tilde{\mathbf{X}}, \quad \mathbf{X}^\top \tilde{\mathbf{X}} = \mathbf{X}^\top \mathbf{X} - \operatorname{diag}(\mathbf{s}),
\end{equation}

\noindent where $\mathbf{s} \in \mathbb{R}^p$ is a vector of non-negative constants, and $\operatorname{diag}(\mathbf{s})$ denotes a diagonal matrix with entries from $\mathbf{s}$. The knockoff variables serve as negative controls, replicating the correlation structure among variables while being conditionally independent of the response given the original variables \cite{dai2016knockoff}.

By augmenting the design matrix to $[\mathbf{X} \ \tilde{\mathbf{X}}]$, feature importance statistics $W_j$ for $j = 1, \dots, p$ are computed, often based on the difference between the original and coefficients of the knockoff variables in a fitted model. Variables are selected by applying a data-driven threshold $\tau$ such that the estimated FDR,

\begin{equation}\label{eq:fdr_estimate}
\widehat{\text{FDR}} = \frac{\# \{ j : W_j \leq -\tau \}}{\# \{ j : W_j \geq \tau \} \vee 1},
\end{equation}

\noindent is controlled at a desired level $q$.

Extensions of the Knockoff framework have been developed for various contexts, including: \emph{Group Sparsity}: Selecting groups of variables simultaneously \cite{ren2022derandomized}, \emph{Structural Sparsity}: Accounting for sparsity induced by transformations \cite{cao2021controlling}, and \emph{High-Dimensional Regression}: Adapting to settings where $p \geq n$ \cite{machkour2021terminating}. Methods such as \emph{split Knockoffs} \cite{cao2021controlling} and \emph{Gaussian mirrors} \cite{xing2019controlling} have been proposed to enhance power and accommodate different correlation structures. These advances demonstrate the versatility of Knockoffs in controlling FDR across diverse statistical models.

\subsection{FDR Control for Support Recovery}

Controlling the false discovery rate is critical for reliable support recovery in high-dimensional settings \cite{candes2015slope}. The Sorted $\ell_1$ Penalized Estimation (SLOPE) method \cite{abramovich2005adapting} introduces a sorted $\ell_1$ penalty to the regression problem:

\begin{equation}\label{eq:slope}
\min_{\mathbf{x}} \frac{1}{2}\|\mathbf{y} - \mathbf{A}\mathbf{x}\|_2^2 + \sum_{j=1}^n \lambda_j |x|_{(j)},
\end{equation}

\noindent where $|x|_{(1)} \geq |x|_{(2)} \geq \dots \geq |x|_{(n)}$ are the ordered absolute values of the coefficients, and $\boldsymbol{\lambda} = (\lambda_1, \dots, \lambda_n)$ is a sequence of tuning parameters designed to control the FDR. The \emph{T-Rex} algorithm \cite{machkour2024sparse} offers FDR guarantees by integrating adaptive thresholding based on estimated noise levels and signal sparsity. However, these methods are often tailored for overdetermined or exactly determined systems and may not directly apply to the underdetermined nature of CS measurement models.

The synergy between FDR control and sparse estimation has been further explored in multiple hypothesis testing and high-dimensional inference \cite{emery2019controlling}. Yet, applying FDR control mechanisms specifically to CS support recovery remains under-investigated. Our work addresses this gap by adapting the Knockoff framework to the CS context.

\subsection{Discussion}
Our Knockoff-guided framework, \TheName{}, constitutes a novel fusion of compressive sensing, statistical Knockoffs, and FDR control methodologies. By extending the knockoff construction to the CS measurement model, we develop an approach that provides, for the first time, rigorous FDR control on support recovery within compressive sensing. Unlike traditional CS methods that primarily aim for accurate signal reconstruction or support estimation under certain conditions (e.g., Restricted Isometry Property, coherence constraints) \cite{baron2008bayesian,kerkouche2020compression}, our method explicitly controls the expected proportion of false discoveries in the estimated support set. This is achieved through the tailored construction of measurement knockoffs that serve as control variables analogous to those in the original Knockoff framework.

By incorporating statistical guarantees on the FDR, our approach enhances the reliability of support recovery, which is paramount in applications where false positives can lead to significant consequences. Moreover, this integration opens new avenues for applying advanced statistical inference techniques within the realm of compressive sensing, promoting more robust and interpretable solutions in high-dimensional signal recovery problems. 

\section{\TheName{}: Knockoff-Guided Compressive Sensing Frameworks and Algorithms  }
We present a novel framework that leverages statistical Knockoff filters to achieve guaranteed sparse signal recovery in compressive sensing. Our approach separates the support recovery and signal estimation phases, enabling precise false discovery rate (FDR) control while maintaining computational efficiency.

\subsection{Framework Overview}
Consider the standard compressive sensing model in Equation~\ref{eq:cs_model}. Our framework consists of three main phases. Initially, we construct knockoff measurements for the measurement matrix, adapting statistical Knockoff filters to the compressive sensing setting. Subsequently, we perform support recovery with FDR control, identifying the support of the sparse signal using the knockoff measurements while ensuring that the false discovery rate does not exceed a specified level. Finally, we estimate the signal on the identified support to reconstruct the sparse signal $\mathbf{x}$. The main steps of \TheName{} are outlined in Algorithm \ref{alg:knockoff_cs}.

\begin{algorithm}[!t]
\caption{\texttt{KnockoffCS}: Knockoff-Guided Compressive Sensing}
\label{alg:knockoff_cs}
\begin{algorithmic}[1]
\Require Measurements $\mathbf{y}$, measurement matrix $\mathbf{A}$, target FDR level $q$
\Ensure Reconstructed sparse signal $\hat{\mathbf{x}}$
\Statex\textcolor{blue}{Generate a knockoff matrix $\tilde{\mathbf{A}}$ based on the measurement matrix $\mathbf{A}$}
\State Compute $\mathbf{\Sigma} \gets \mathbf{A}^\top\mathbf{A}$
\State Set $\mathbf{s} \gets \min\{1, 2\lambda_{\min}(\mathbf{\Sigma})\}\mathbf{1}$
\State Construct $\mathbf{C} \gets \mathbf{\Sigma}^{-1/2}(\mathbf{s} \odot \mathbf{\Sigma}^{1/2})$
\State Generate $\tilde{\mathbf{A}} \gets \mathbf{A}(\mathbf{I} - \mathbf{C})$
\Statex\textcolor{blue}{Identify the support set $\hat{S}$ from measurements $\mathbf{y}$ using the measurement matrix $\mathbf{A}$, the knockoff matrix $\tilde{\mathbf{A}}$, and target FDR level $q$}
\State Compute $W_j \gets |[\mathbf{A}^\top\mathbf{y}]_j| - |[\tilde{\mathbf{A}}^\top\mathbf{y}]_j|$ for all $j$
\State Sort $|W_j|$ in descending order: $|W_{(1)}| \geq \cdots \geq |W_{(n)}|$
\State Set $T\gets\underset{t>0}{\mathrm{argmax}}\left\{\mathrm{adjust}\left(\frac{ \left| \{ j : W_j \leq -t \} \right| }{ \left| \{ j : W_j \geq t \} \right| \vee 1 }\right) \leq q\right\}$ \Comment{\textcolor{blue}{p-value adjustment for threshold setting}}
\State Set $\hat{S} \gets \{j: W_j \geq T\}$\Comment{\textcolor{blue}{p-value thresholding for significant support discoveries}}
\Statex\textcolor{blue}{Reconstruct signal $\hat{\mathbf{x}}$ from measurements $\mathbf{y}$ and the measurement matrix $\mathbf{A}$ using the support set $\hat{S}$ via linear regression or Ridge}
\State Compute $\hat{\mathbf{x}}_{\hat{S}} \gets (\mathbf{A}_{\hat{S}}^\top\mathbf{A}_{\hat{S}})^{-1}\mathbf{A}_{\hat{S}}^\top\mathbf{y}$\ or\ $\hat{\mathbf{x}}_{\hat{S}} \gets (\mathbf{A}_{\hat{S}}^\top\mathbf{A}_{\hat{S}}+\lambda\mathbf{I})^{-1}\mathbf{A}_{\hat{S}}^\top\mathbf{y}$
\State Set $\hat{\mathbf{x}}_{\hat{S}^c} \gets \mathbf{0}$ \Comment{\textcolor{blue}{Set zero to the complement of the support set}}
\State \Return $\hat{\mathbf{x}}\gets\hat{\mathbf{x}}_{\hat{S}}\cup\hat{\mathbf{x}}_{\hat{S}^c}$ \Comment{\textcolor{blue}{Return the full-length vector as the recovered sparse signal}}
\end{algorithmic}
\end{algorithm}

\subsection{Knockoff Construction for Compressed Sensing}
The key innovation of our framework lies in adapting Knockoff filters to the compressive sensing scenario. Given a known measurement matrix $\mathbf{A}$, this step aims to construct a fixed-X knockoff~\cite{Barber2015} measurement matrix $\tilde{\mathbf{A}}$ such that:
\begin{equation}
    \begin{bmatrix} 
        \mathbf{A}^\top\mathbf{A} & \mathbf{A}^\top\tilde{\mathbf{A}} \\
        \tilde{\mathbf{A}}^\top\mathbf{A} & \tilde{\mathbf{A}}^\top\tilde{\mathbf{A}}
    \end{bmatrix} = 
    \begin{bmatrix}
        \mathbf{\Sigma} & \mathbf{\Sigma} - \mathbf{S} \\
        \mathbf{\Sigma} - \mathbf{S} & \mathbf{\Sigma}
    \end{bmatrix}
\end{equation}
where $\mathbf{\Sigma} = \mathbf{A}^\top\mathbf{A}$ and $\mathbf{S}$ is a diagonal matrix satisfying $0 \preceq \mathbf{S} \preceq 2\mathbf{\Sigma}$. 

As show in lines 1--4 of Algorithm~\ref{alg:knockoff_cs}, to construct the knockoff matrix $\tilde{\mathbf{A}}$, we proceed as follows. First, we compute $\mathbf{\Sigma} = \mathbf{A}^\top\mathbf{A}$. We then set $\mathbf{s} = \min\{1, 2\lambda_{\min}(\mathbf{\Sigma})\}\mathbf{1}$, where $\lambda_{\min}(\mathbf{\Sigma})$ denotes the smallest eigenvalue of $\mathbf{\Sigma}$ and $\mathbf{1}$ is the vector of ones. Next, we construct the matrix $\mathbf{C} = \mathbf{\Sigma}^{-1/2} \left( \mathbf{s} \odot \mathbf{\Sigma}^{1/2} \right)$, where $\odot$ denotes element-wise multiplication. Finally, we generate the knockoff matrix by computing $\tilde{\mathbf{A}} = \mathbf{A} \left( \mathbf{I} - \mathbf{C} \right)$.

Note that, compared to the model-X Knockoff~\cite{candes2018panning}, which requires the joint distribution of the features to generate knockoffs, the fixed-X Knockoff directly uses the known measurement matrix, or namely the design matrix, \(\mathbf{A}\) to construct knockoff matrices that preserve its correlation structure. This advancement makes the fixed-X Knockoff more computationally efficient and suitable for scenarios with a known, fixed $\mathbf{A}$.

\subsection{FDR-Controlled Support Recovery} \label{subsection:FDR_control_support_recovery}
With the knockoff measurements in place, we proceed to construct feature statistics for support recovery. For each index $j \in \{1, 2, \dots, n\}$, we define the statistic:
\begin{equation} \label{eq:Wstatistics}
    W_j = \left| [\mathbf{A}^\top\mathbf{y}]_j \right| - \left| [\tilde{\mathbf{A}}^\top\mathbf{y}]_j \right|
\end{equation}
To select the support $\hat{S}$ while controlling the false discovery rate at a target level $q$, we determine a data-dependent threshold $T$. This threshold is the smallest value $t > 0$ satisfying:
\begin{equation}\label{eq:Tstatistics}
    \mathrm{adjust}\left(\frac{ \left| \{ j : W_j \leq -t \} \right| }{ \left| \{ j : W_j \geq t \} \right| \vee 1 }\right) \leq q \,
\end{equation}
where $\mathrm{adjust}(\cdot)$ refers to the operator for p-value adjustment. A simple implementation could be $\mathrm{adjust}(v)=v$ or using the Benjamini-Hochberg procedure~\cite{benjamini1995controlling}. Here, the numerator counts the number of indices for which $W_j$ is less than or equal to $-t$, and the denominator counts the number of indices for which $W_j$ is greater than or equal to $t$ (using $\vee 1$ to ensure the denominator is at least 1). We then select the support set as 
\begin{equation} \label{eq:estimated_support_set}
    \hat{S} = \{ j : W_j \geq T \}\ .
\end{equation}
This procedure ensures that the expected proportion of false discoveries among the selected variables does not exceed the specified level $q$. In lines 5--8 of Algorithm~\ref{alg:knockoff_cs}, we show the procedure to identify the support set under FDR control.

\subsection{Signal Estimation}
After identifying the support $\hat{S}$, we estimate the signal $\mathbf{x}$ by solving a least-squares problem restricted to the selected support. Specifically, we first compute:
\begin{equation}\label{eq:estimator}
    \hat{\mathbf{x}}_{\hat{S}} = \left( \mathbf{A}_{\hat{S}}^\top\mathbf{A}_{\hat{S}} \right)^{-1} \mathbf{A}_{\hat{S}}^\top\mathbf{y},
\end{equation}
where $\mathbf{A}_{\hat{S}}$ denotes the columns of $\mathbf{A}$ corresponding to the support $\hat{S}$. For this step, to address numerical stability issues that may arise due to ill-conditioned matrices, we also can incorporate a Ridge estimator~\cite{mcdonald2009ridge} by adding a small positive parameter $\lambda$ to the diagonal of the matrix to be inverted:
\begin{equation}
    \hat{\mathbf{x}}_{\hat{S}} = \left( \mathbf{A}_{\hat{S}}^\top\mathbf{A}_{\hat{S}} + \lambda \mathbf{I} \right)^{-1} \mathbf{A}_{\hat{S}}^\top\mathbf{y}
\end{equation}

Later, we use $\hat{S}^c$ as the complement of $\hat{S}$ and set $\hat{\mathbf{x}}_{\hat{S}^c} = \mathbf{0}$. We combine these two vectors $\hat{\mathbf{x}}_{\hat{S}}$ and $\hat{\mathbf{x}}_{\hat{S}^c}$, subject to the dimensions in the sparse signal space $\mathbb{R}^n$, to construct the sparse signal, as follows: 
\begin{equation}
    \hat{\mathbf{x}}=\hat{\mathbf{x}}_{\hat{S}}\cup\hat{\mathbf{x}}_{\hat{S}^c}
\end{equation}
In terms of computational efficiency for large-scale problems, we utilize conjugate gradient methods for solving linear systems, which are effective for large, sparse matrices. 
Lines 9-11 of Algorithm~\ref{alg:knockoff_cs} demonstrate the procedure of signal estimation.

\subsection{Discussions}
Several practical considerations enhance the performance of the framework as follows.

\begin{remark}[Practical Considerations] To use our proposed algorithms in practice, we make the following remarks:
\begin{itemize}
    \item \textbf{Parameter Tuning} - Parameter selection is crucial for the performance of \TheName{}. The FDR level $q$ should be chosen based on the specific requirements of the application, balancing the trade-off between discovery and false positives. The knockoff strength parameter $\mathbf{s}$ can be adjusted to optimize the power of the test, and the regularization parameter $\lambda$ should be selected to ensure numerical stability without introducing significant bias into the estimation.    
    \item \textbf{Computational Complexity} - The proposed framework has a computational complexity of $O(mn + n\log n)$ for support recovery and $O(ms^2)$ for signal estimation, which is comparable to standard LASSO implementations. The $O(mn)$ term arises from computing $\mathbf{A}^\top\mathbf{y}$ and $\tilde{\mathbf{A}}^\top\mathbf{y}$, while the $O(n\log n)$ term comes from sorting the statistics $W_j$.

    \item \textbf{Extensions} - The framework naturally extends to various settings. For group sparsity, group knockoffs can be employed to simultaneously consider groups of variables, enhancing detection power when variables are naturally grouped. In scenarios with structured sparsity, Knockoff filters can be adapted to account for known structures in the sparsity pattern, such as tree structures or spatial patterns. For dynamic settings, sequential Knockoffs can be utilized to handle data arriving over time, enabling online support recovery with FDR control.

\end{itemize}
\end{remark}

\section{Main Results}

Classical compressed sensing approaches based on $\ell_1$ minimization, such as the Lasso, provide guarantees for signal recovery but do not offer precise control over false discoveries in support estimation. Our newly proposed framework \TheName{} combines Knockoff filter, a statistical framework for controlled variable selection, with compressive sensing. This section presents the theoretical foundations of the proposed framework, and formally establishes that it achieves two desirable properties: accurate signal reconstruction and rigorous FDR control in support recovery.

Consider the compressive sensing model in Equation~\ref{eq:cs_model}, where $\mathbf{x} \in \mathbb{R}^n$ is the true sparse signal, $\mathbf{y} \in \mathbb{R}^m$ is the observation vector, and $\mathbf{w} \in \mathbb{R}^m$ is the noise vector. Our analysis is based on the following settings (as formally detailed in Assumption~\ref{assumption:A1}):

\begin{enumerate}
    \item \textbf{Signal Properties:} The signal $\mathbf{x}$ is sparse, with $|\mathbf{x}|_0 \leq s$, where $s \le m$.
    
    \item \textbf{Noise Assumptions:} The noise vector $\mathbf{w}$ is assumed to follow a Gaussian distribution, specifically, $\mathbf{w} \sim \mathcal{N}(0, \frac{\sigma^2}{m} \mathbf{I}_m)$, where \(\sigma/\sqrt{m}\) is the noise standard deviation and \(m\) is the number of measurements.
    
   \item \textbf{Measurement Matrix:} The measurement matrix \(\mathbf{A} \in \mathbb{R}^{m \times n}\) maps the signal \(\mathbf{x}\) to the observations \(\mathbf{y}\) via the model \(\mathbf{y} = \mathbf{A} \mathbf{x} + \mathbf{w}\). The measurement matrix and its knockoff matrix are assumed to have a lower bound on the relative difference between the true and knockoff measurements with respect to their interaction with the signal.
\end{enumerate}

Our main theoretical result guarantees both support recovery and signal reconstruction, as summarized in the following proposition.

\begin{proposition}[Recovery Guarantees]
Our knockoff-guided procedure, \TheName{}, satisfies the following:

\begin{enumerate}
\item \textbf{FDR Control in Support Discoveries}:  
The procedure controls the False Discovery Rate (FDR) in support discoveries. Given a fixed target FDR level \(q \in (0,1)\), the estimated support set \(\hat{S}\) is determined by:
\begin{equation}
    \hat{S} = \{ j : W_j \geq T \},
\end{equation}
where \(W_j\) is the knockoff statistic for the \(j\)-th variable, and \(T\) is a threshold based on \(q\) (as defined in \eqref{eq:Tstatistics}). The False Discovery Proportion (FDP) is:
\begin{equation}\label{fdp_def}
\text{FDP}(\hat{S}) = \frac{|\hat{S} \setminus S|}{|\hat{S}| \vee 1},
\end{equation}
where \(S\) is the true support of \(\mathbf{x}\), \(\hat{S} \setminus S\) represents the set of elements in \(\hat{S}\) but not in \(S\), and \(|\cdot|\) denotes the cardinality of a set. The FDR, the expected value of FDP, is controlled as:
\begin{equation}\label{fdr_def}
\text{FDR}(\hat{S}) = \mathbb{E}[\text{FDP}(\hat{S})] \leq q.
\end{equation}
Additionally, the probability of correctly identifying the true support \(S\) is bounded as:
\begin{equation}
\mathbb{P}(S \subseteq \hat{S}) \geq 1 - 4s \exp\left(-\frac{(\delta - T)^2}{8\sigma^2} m\right),
\end{equation}

\item \textbf{Reconstruction Error Bound}:  
Assuming that the \TheName{} estimator \(\hat{\mathbf{x}}\) is defined as in equation (\ref{eq:estimator}), and the measurement matrix \(\mathbf{A}\) satisfies the conditions on minimal eigenvalue, finite correlation and orthogonality among certain columns (as detailed in Assumption~\ref{assumption:A2}). Additionally, we assume \(|\hat{S}| \le m\). Then the following reconstruction error bound for the estimator \(\hat{\mathbf{x}}\) holds with probability at least \(1- 4s \exp\left(-\frac{(\delta - T)^2}{8\sigma^2} m\right) - \exp\left(-\frac{C s }{8}\log n\right) - \exp\left( -\frac{C |S_f| }{8}\log n \right)\):
\begin{equation}
\|\mathbf{\hat{x}}-\mathbf{x}\|_2 \leq \frac{\sqrt{C} \sigma \left( \sqrt{s} + \sqrt{|S_f|} \right)} {\left( \kappa_{\min}^2 - \gamma  \right)}\sqrt{\frac{\log n}{m}}.
\end{equation}
Here, the target FDR level \(q\) affects the trade-off between the upper bound and the probability of satisfying it. Lowering \(q\) reduces the estimated support set \(\hat{S}\), which in turn decreases \(|S_f|\), leading to a tighter reconstruction error bound, but also lowering the probability of satisfying the bound.

\item \textbf{Refined Reconstruction Error Bound}:
The reconstruction error bound can be further refined. Specifically, there exist constants \( C > \frac{2}{\log 2} \) such that, for a target FDR control level of \( q = c \cdot \frac{s}{m} \), where \( c \) is a constant ensuring \( q \in (0,1) \), and for any constant \( \eta \in (0,1) \), the \TheName{} estimator \( \hat{\mathbf{x}} \) satisfies:
\begin{equation}
\|\mathbf{\hat{x}}-\mathbf{x}\|_2\le
\frac{\sqrt{C}}{\kappa_{\min}^2 - \gamma}\left(1+\sqrt{\frac{c}{\eta}}\right) \sigma \sqrt{\frac{s\log n}{m}}
\end{equation}
with probability at least \(1- 4s \exp\left(-\frac{(\delta - T)^2}{8\sigma^2} m\right) - \exp\left(-\frac{C s }{8}\log n\right) - \exp\left( -\frac{C |S_f| }{8}\log n \right)- \eta\).

\item \textbf{Expected Reconstruction Error and FDR Level}:  
Provided the reconstruction error bound holds, for $q = O(\frac{s}{m})$, the \TheName{} estimator \(\hat{\mathbf{x}}\) satisfies the following expected reconstruction error bound:
\begin{equation}
\mathbb{E}\bigl[\|\hat{\mathbf{x}} - \mathbf{x}\|_2\bigr] 
\le O\left(\sigma \sqrt{\frac{s \log n}{m}}\right).
\end{equation}

\end{enumerate}
\end{proposition}
These results theoretically ensure both accurate signal reconstruction and rigorous control of false discoveries in the context of \TheName{}.

\subsection{Support Recovery via FDR Control}
Our first result establishes the connection between FDR control and support recovery, based on the following conditions:

\begin{assumption}[Knockoff Procedure and Signal Recovery Conditions]\label{assumption:A1}
We assume the measurement matrix $\mathbf{A} \in \mathbb{R}^{m \times n}$, its knockoff matrix \(\tilde{\mathbf{A}}\), the true signal $\mathbf{x} \in \mathbb{R}^n$, and the noise vector $\mathbf{w} \in \mathbb{R}^m$ satisfy the following conditions:

\begin{itemize}
    \item \textbf{(A1.1) Measurement Matrix:} \\
    (a) \textit{Standardization}: $\|\mathbf{A}_j\|_2^2 = \|\tilde{\mathbf{A}}_j\|_2^2 = 1$ for all $j \in \{1, \dots, n\}$, where \(\mathbf{A}_j, \tilde{\mathbf{A}}_j\) denotes the \(j\)-th column vector of matrix \(\mathbf{A}\) and \(\tilde{\mathbf{A}}\) respectively.\\
    (b) \textit{Relative difference in interaction with the signal}: The knockoff matrix $\tilde{\mathbf{A}} \in \mathbb{R}^{m \times n}$ is constructed such that \(\exists \delta > 0\), for every \(j\) in the true support set \(S\), the following condition holds:
    \begin{equation}
    |\alpha_j| \geq |\beta_j| + \delta,
    \end{equation}
    where
    \begin{equation} \label{eq:def_of_alpha}
    \alpha_j = [\mathbf{A}^\top \mathbf{A} \mathbf{x}]_j \quad \text{and} \quad \beta_j = [\tilde{\mathbf{A}}^\top \mathbf{A} \mathbf{x}]_j,
    \end{equation}
    Here, \(\left[\cdot\right]_j\) represents the \(j\)-th component of a vector.
    \item \textbf{(A1.2) Sparsity:} The true signal $\mathbf{x} \in \mathbb{R}^n$ is assumed to be $s$-sparse, i.e., $\|\mathbf{x}\|_0 \leq s$, where \(s \le m\).
    
    \item \textbf{(A1.3) Noise:} The noise vector $\mathbf{w} \in \mathbb{R}^m$ is assumed to follow a Gaussian distribution, specifically, $\mathbf{w} \sim \mathcal{N}(0, \frac{\sigma^2}{m}\mathbf{I}_m)$.
\end{itemize}
\end{assumption}

\begin{theorem}[Support Recovery via FDR Control] \label{thm:SupportRecoveryFDRControl}
Suppose the conditions outlined in Assumption~\ref{assumption:A1} hold. For any target False Discovery Rate (FDR) level $q \in (0,1)$, \TheName{} satisfies the following FDR control condition:
\begin{equation}\label{eq:fdr_control}
    \text{FDR} = \mathbb{E}\left[\frac{|\hat{S} \setminus S|}{|\hat{S}| \vee 1}\right] \leq q,
\end{equation}
where $S = \text{supp}(\mathbf{x})$ is the true support of the signal, $\hat{S} = \left\{ j : W_j \geq T \right\}$ is the estimated support set defined in (\ref{eq:estimated_support_set}), \(\hat{S} \setminus S\) represents the set of elements in \(\hat{S}\) but not in \(S\), and \(|\cdot|\) denotes the cardinality of a set. Furthermore, the probability of correctly identifying the true support is bounded below as
\begin{equation}\label{eq:support_inclusion}
    \mathbb{P}\left(S \subseteq \hat{S}\right) \geq 1 - 4s \exp\left(-\frac{(\delta - T)^2}{8\sigma^2} m\right),
\end{equation}
where $m$ is the number of measurements.
\end{theorem}

\begin{proof}[Proof of Theorem \ref{thm:SupportRecoveryFDRControl}]
The proof of inequality \eqref{eq:fdr_control} follows directly from the methodology presented in \cite{Barber2015}.
We now proceed to establish inequality \eqref{eq:support_inclusion}. Let \(j \in S\). According to equation (\ref{eq:Wstatistics}) in section~\ref{subsection:FDR_control_support_recovery}, the knockoff statistic \(W_j\) is defined as
\begin{equation}
W_j = \left| [\mathbf{A}^\top\mathbf{y}]_j \right| - \left| [\tilde{\mathbf{A}}^\top\mathbf{y}]_j \right| 
= \left| [\mathbf{A}^\top(\mathbf{Ax}+\mathbf{w})]_j \right| - \left| [\tilde{\mathbf{A}}^\top(\mathbf{Ax}+\mathbf{w})]_j \right|
\end{equation}
\begin{equation}
 = \left| [\mathbf{A}^\top\mathbf{Ax}]_j + [\mathbf{A}^\top \mathbf{w}]_j \right| - \left| [\tilde{\mathbf{A}}^\top\mathbf{Ax}]_j + [\tilde{\mathbf{A}}^\top\mathbf{w}]_j \right| = |\alpha_j + \gamma_j| - |\beta_j + \tilde{\gamma}_j|.
\end{equation}
Here, \(\left[\cdot\right]_j\) represents the \(j\)-th component of a vector. The definitions of \(\alpha_j\) and \(\beta_j\) are as in (\ref{eq:def_of_alpha}). Additionally, we define
\begin{equation}
\gamma_j = \left[ \mathbf{A}^\top \mathbf{w} \right]_j = \sum_i \mathbf{A}_{i,j} \mathbf{w}_i, \quad \tilde{\gamma}_j = \tilde{\mathbf{A}}_j^\top \mathbf{w} = \sum_i \tilde{\mathbf{A}}_{i,j} \mathbf{w}_i.
\end{equation}
In these expressions, \(\mathbf{A}_{i,j}\) and \(\mathbf{\tilde{A}}_{i,j}\) are the elements in the \(i\)-th row and \(j\)-th column of the matrices \(\mathbf{A}\) and \(\mathbf{\tilde{A}}\), respectively, and \(\mathbf{w}_i\) denotes the \(i\)-th element of vector \(\mathbf{w}\).

Since $\mathbf{w}\sim\mathcal{N}(0, \frac{\sigma^2}{m}I_m)$ and $\Sigma_i\mathbf{A}^2_{i,j} = \|\mathbf{A}_j\|^2_2 = 1$, we have $\gamma_j \sim \mathcal{N}(0, \frac{\sigma^2}{m})$. Similarly, $\tilde{\gamma}_j \sim \mathcal{N}(0, \frac{\sigma^2}{m})$.

Given that $\alpha_j \ge \beta_j + \delta$ for $j \in S$, we can establish a lower bound for $W_j$ as follows:
\begin{equation}
W_j = |\alpha_j + \gamma_j| - |\beta_j + \tilde{\gamma}_j| \\
\end{equation}
\begin{equation}
\ge |\alpha_j| - |\gamma_j| - |\beta_j| - |\tilde{\gamma}_j| \quad \text{(by the triangle inequality)} \\
\end{equation}
\begin{equation}
\ge \delta - \left( |\gamma_j| + |\tilde{\gamma}_j| \right).
\end{equation}
Therefore, the condition $W_j \ge T$ is satisfied provided that
\begin{equation}
|\gamma_j| + |\tilde{\gamma}_j| \le \delta - T.
\end{equation}
Thus we have
\begin{equation} \label{eq:W_j_bound}
    \mathbb{P}(W_j \ge T) \le \mathbb{P}(|\gamma_j|+|\tilde{\gamma}_j|\le\delta-T).
\end{equation}
Next, consider the tail bound for a Gaussian random variable $X \sim \mathcal{N}(0, \frac{\sigma^2}{m})$:
\begin{equation}
\mathbb{P}(|X| > t) \le 2 \exp\left( -\frac{mt^2}{2\sigma^2} \right).
\end{equation}
Applying this bound to both $\gamma_j$ and $\tilde{\gamma}_j$, we obtain
\begin{equation}
\mathbb{P}\left( |\gamma_j| > t \right) \le 2 \exp\left( -\frac{mt^2}{2\sigma^2} \right), \quad \mathbb{P}\left( |\tilde{\gamma}_j| > t \right) \le 2 \exp\left( -\frac{mt^2}{2\sigma^2} \right).
\end{equation}
To bound the probability $\mathbb{P}\left( |\gamma_j| + |\tilde{\gamma}_j| > \delta - T \right)$, we apply the union bound:
\begin{equation}
\mathbb{P}\left( |\gamma_j| + |\tilde{\gamma}_j| > \delta - T \right) \le \mathbb{P}\left(\{|\gamma_j|>\frac{\delta-T}{2}\}\cup\{|\tilde{\gamma}_j|>\frac{\delta-T}{2}\}\right)
\end{equation}
\begin{equation}
=\mathbb{P}\left( |\gamma_j| > \frac{\delta - T}{2} \right) + \mathbb{P}\left( |\tilde{\gamma}_j| > \frac{\delta - T}{2} \right)-\mathbb{P}\left(\{|\gamma_j|>\frac{\delta-T}{2}\}\cap\{|\tilde{\gamma}_j|>\frac{\delta-T}{2}\}\right) \\
\end{equation}
\begin{equation}
\le\mathbb{P}\left( |\gamma_j| > \frac{\delta - T}{2} \right) + \mathbb{P}\left( |\tilde{\gamma}_j| > \frac{\delta - T}{2} \right) \\
\end{equation}
\begin{equation}
\le 2 \exp\left( -\frac{m(\delta - T)^2}{8\sigma^2} \right) + 2 \exp\left( -\frac{m(\delta - T)^2}{8\sigma^2} \right) \\
\end{equation}
\begin{equation}
= 4 \exp\left( -\frac{m(\delta - T)^2}{8\sigma^2} \right).
\end{equation}
Recalling \eqref{eq:W_j_bound} and applying the union bound over all $j \in S$, we have
\begin{equation}
\mathbb{P}\left( \exists j \in S : W_j < T \right) \le \sum_{j \in S} \mathbb{P}\left( W_j < T \right) \le 4s \exp\left( -\frac{m(\delta - T)^2}{8\sigma^2} \right),
\end{equation}
where $s = |S|$.
Thus, we obtain
\begin{equation}
\mathbb{P}(S \subseteq \hat{S}) = 1 - 4s \exp\left( -\frac{m(\delta - T)^2}{8\sigma^2} \right).
\end{equation}
This concludes the proof.
\end{proof}

\subsection{FDR Control and Reconstruction Error}
This section presents the main recovery guarantee and provides several important assumptions and theorems related to reconstruction error for the estimated signal \(\hat{\mathbf{x}}\). Throughout this and the following sections, we use \( \mathbf{A}_{S'} \) to denote the submatrix of \( \mathbf{A} \) consisting of the columns indexed by any index set \( S' \subseteq \{1, \dots, n\} \).  

\begin{assumption}[Measurement Matrix Properties and Support Set Bound]\label{assumption:A2}
We assume the following conditions hold for the measurement matrix $\mathbf{A} \in \mathbb{R}^{m \times n}$ and the estimated support set $\hat{S}$:
\begin{itemize}
    \item \textbf{(A2.1) Restricted Eigenvalue (RE) Condition:} Restricted Eigenvalue (RE) Condition: There exist constants \( \kappa_{\max}, \kappa_{\min} > 0 \) such that for any vector \( \mathbf{u} \in \mathbb{R}^n \) with support \( S' = \mathrm{supp}(\mathbf{u}) \) and \( |S'| \leq m \), the following condition holds:
    \[
    \kappa_{\min} \|\mathbf{u}_{S'}\|_2 \le \|\mathbf{A}_{S'} \mathbf{u}_{S'}\|_2 \le \kappa_{\max} \|\mathbf{u}_{S'}\|_2,
    \]
    
    \item \textbf{(A2.2) Mutual Coherence Condition:} There exists a finite constant \( \gamma \) (with \( 0 < \gamma < \infty \)) that quantifies the limited correlation between the columns indexed by \( S_c \) and \( S_f \) as follows:
    \[
    \|\mathbf{A}_{S_c}^\top \mathbf{A}_{S_f}\|_2 \leq \gamma,
    \]
    where $S_c = \hat{S} \cap S$ is the correctly identified support set and $S_f = \hat{S} \setminus S$ is the false discovery support set.
    \item \textbf{(A2.3) Orthogonality Condition}: The columns of $\mathbf{A}_{S}$ and $\mathbf{A}_{S_f}$ are orthogonal.
    \item \textbf{(A2.4) Bound on Estimated Support Size:} The estimated support set $\hat{S}$ satisfies $|\hat{S}| \leq m$.
\end{itemize}

We assume the measurement matrix $\mathbf{A} \in \mathbb{R}^{m \times n}$ and the estimated support set $\hat{S}$ satisfy the above conditions, which ensure the reconstruction error bound of the estimated signal.
\end{assumption}

\begin{theorem}[Reconstruction Error Bound for Estimated Signal]
\label{thm:fdr_error1}
Let $S$ and $\hat{S}$ represents the true and estimated support sets, respectively and $|S|=s$. Denote the correctly identified support set as $S_c = \hat{S} \cap S$, the false discovery support set as $S_f = \hat{S} \setminus S$, and the missed support set as $S_m = S \setminus \hat{S}$. 

Under Assumptions~\ref{assumption:A1}-\ref{assumption:A2}, for any target FDR control level \( q \in (0,1) \), there exist constants \(C > \frac{2}{\log2} \) such that the \TheName{} estimator \( \hat{\mathbf{x}} \) defined in (\ref{eq:estimator}) satisfies the following reconstruction error bound, with probability at least \(1- 4s \exp\left(-\frac{(\delta - T)^2}{8\sigma^2} m\right) - \exp\left(-\frac{C s }{8}\log n\right) - \exp\left( -\frac{C |S_f| }{8}\log n \right)\):
\begin{equation}\label{eq:Int_result}
\|\mathbf{\hat{x}}-\mathbf{x}\|_2 \leq \frac{\sqrt{C} \sigma \left( \sqrt{s} + \sqrt{|S_f|} \right)} {\left( \kappa_{\min}^2 - \gamma  \right)}\sqrt{\frac{\log n}{m}}
\end{equation}
 where \(\sigma/\sqrt{m}\) is the noise standard deviation.
\end{theorem}

To prove this, we first consider the following lemma.
\begin{lemma}[Bound on Error Terms] \label{lemma:subexponentialbound}
Under Assumptions~\ref{assumption:A1}-\ref{assumption:A2}, when \(S\subseteq\hat{S}\) holds (with probability at least \(1 - 4s \exp\left(-\frac{(\delta - T)^2}{8\sigma^2} m\right)\)), then there exists a constant \(C > \frac{2}{\log 2}\) such that the following bound holds with probability at least \(1 - \exp\left(-\frac{C s }{8}\log n\right) - \exp\left( -\frac{C |S_f|}{8}\log n \right)\), for the projections of the noise vector $\mathbf{w}$ onto the submatrices $\mathbf{A}_{S_c}$ and $\mathbf{A}_{S_f}$ corresponding to the correct and false discovery support sets:
\begin{equation}\label{eq:lemma_error_bound}
\|\mathbf{A}_{S_c}^\top \mathbf{w}\|_2 + \|\mathbf{A}_{S_f}^\top \mathbf{w}\|_2 \leq \sqrt{C} \left( \sqrt{s} + \sqrt{|S_f|} \right) \sigma \sqrt{\frac{\log n}{m}},
\end{equation}
where \(\sigma/\sqrt{m}\) is the noise standard deviation.
\end{lemma}

\begin{proof}[Proof of Lemma \ref{lemma:subexponentialbound}]
When \(S \subseteq \hat{S}\) holds, $S_c=S\cap\hat S=S$ and thus $|S_c|=s$. We begin by noting that, due to the orthogonality between the columns of \( \mathbf{A}_{S} \) and \( \mathbf{A}_{S_f} \), the random variables \( A_i^\top \mathbf{w} \) are independent for \( i \in S \) and \( i \in S_f \). Moreover, since \( A_i^\top \mathbf{w} \) follows a normal distribution, it is sub-Gaussian. Let \( X_i = \left( A_i^\top \mathbf{w} \right)^2 \), which is sub-exponential and follows a \( \frac{\sigma^2}{m} \chi^2(1) \) distribution.

To use Bernstein’s inequality to bound \( \|\mathbf{A}_{S_c}^\top \mathbf{w}\|_2 \) and \( \|\mathbf{A}_{S_f}^\top \mathbf{w}\|_2 \), we seek constants \( \nu^2, b > 0 \) such that for all \( |t| \leq \frac{1}{b} \),
\begin{equation}\label{eq:aim_of_subexponentialbound}
    \mathbb{E}\left[ \exp\left(t \left( X_i - \mathbb{E}[X_i] \right)\right) \right] \leq \exp\left( \frac{\nu^2 t^2}{2} \right).
\end{equation}

We proceed by computing the expectation:
\begin{equation}
    \mathbb{E}\left[ \exp\left(t \left( X_i - \mathbb{E}[X_i] \right)\right) \right] = \mathbb{E}\left[ \exp \left( t\left( A_i^\top \mathbf{w} \right)^2 \right) \right] \exp \left( -\frac{t \sigma^2}{m} \right) \\
\end{equation}
\begin{equation}
    = \left( 1 - \frac{2t \sigma^2}{m} \right)^{-1/2} \exp \left( -\frac{t \sigma^2}{m} \right).
\end{equation}
This follows from the moment-generating function of the chi-squared distribution. When \( |\eta| < \frac{1}{2} \), we have 
\begin{equation}
    - \log (1 - \eta) \le \eta + 2 \eta^2.
\end{equation}
For \( |t| < \frac{m}{4\sigma^2} \), it follows that \( \left| \frac{2t\sigma^2}{m} \right| \le \frac{1}{2} \).

Therefore, we obtain the following bound:
\begin{equation}
    -\log \left( 1-\frac{2t\sigma^2}{m}\right) \le \frac{2t\sigma^2}{m} + \frac{8t^2\sigma^4}{m^2}
\end{equation}
\begin{equation}
    \log \mathbb{E}\left(\exp \left(t (X_i  - \mathbb{E}X_i) \right) \right) = \frac{1}{2}\left(\frac{2t\sigma^2}{m} + \frac{8t^2\sigma^4}{m^2}\right) + \left( -\frac{t\sigma^2}{m} \right) = \frac{4t^2\sigma^4}{m^2}
\end{equation}
Thus, for \( |t| \leq \frac{m}{4 \sigma^2} \), we obtain
\begin{equation}
    \mathbb{E}\left[ \exp \left( t \left(X_i - \mathbb{E}X_i\right)\right) \right] \leq \exp \left( \frac{4 t^2 \sigma^4}{m^2} \right).
\end{equation}

Let \( b = \frac{4 \sigma^2}{m} \) and \( \nu^2 = \frac{8 \sigma^4}{m^2} \). By Bernstein's inequality for independent sub-Gaussian variables, we have:
\begin{equation}
    \mathbb{P}\left( | \sum_{i \in S_c} \left( X_i - \mathbb{E}[X_i] \right) | \geq \mu \right) \leq \exp \left( -\frac{1}{2} \min \left\{ \frac{\mu^2}{s \nu^2}, \frac{\mu}{b} \right\} \right).
\end{equation}
Substituting the constants and letting \(\mu = \frac{C s \sigma^2 \log n}{m}\) with \(C > \frac{2}{\log 2}\geq\frac{2}{\log n}\) (as $1\leq m\ll n$ thus $n$ should be an integer greater-equal to $2$), we find:
\[
    \mathbb{P}\left( \|\mathbf{A}_{S_c}^\top \mathbf{w} \|^2_2 \geq \frac{C s \sigma^2 \log n}{m} \right)
\]
\begin{equation}
     \leq \exp \left( -\frac{1}{2} \min \left\{ \frac{C^2 s^2 (\log n)^2 m^2 \sigma^4}{8 \sigma^4 m^2 s}, \frac{C s \log n m \sigma^2}{4 \sigma^2 m} \right\} \right).
\end{equation}
\begin{equation} \label{eq:A_Sc_bound}
    \leq \exp \left( -\frac{C s \log n}{8} \right).
\end{equation}
Similarly, letting \(\mu = \frac{C |S_f| \sigma^2 \log n}{m}\), where \(C\) is the same as in \eqref{eq:A_Sc_bound},
\begin{equation} 
    \mathbb{P}\left( \|\mathbf{A}_{S_f}^\top \mathbf{w} \|^2_2 \geq \frac{C |S_f| \sigma^2 \log n}{m} \right) \leq \exp \left( -\frac{C |S_f| \log n}{8} \right).
\end{equation}
Thus with probability at least \(1-\exp\left( -\frac{C s \log n}{8} \right) - \exp\left( -\frac{C |S_f| \log n}{8} \right)\), we have:
\begin{equation}
    \|\mathbf{A}_{S_c}^T\mathbf{w}\|_2 + \|\mathbf{A}_{S_f}^T\mathbf{w}\|_2 \le \sqrt{C}\left(\sqrt{s}+\sqrt{|S_f|}\right)\sigma\sqrt{\frac{\log n}{m}}
\end{equation}
This concludes the proof.\\\\
Next, we consider the probability expression 
\begin{equation}
1 - 4s \exp\left(-\frac{(\delta - T)^2}{8\sigma^2} m\right) - \exp\left(-\frac{C s \log n}{8}\right) - \exp\left( -\frac{C |S_f| \log n}{8}\right),
\end{equation}
and analyze its behavior for different values of \( |S_f| \). Note that $|S_f| = 0,1,\dots,m-s$:

1. When \( |S_f| = 0 \), \( S_f = \emptyset \), thus the term corresponding to \( |S_f| \) vanishes, and the probability expression simplifies to:
   \begin{equation}
   1 - 4s \exp\left(-\frac{(\delta - T)^2}{8\sigma^2} m\right) - \exp\left(-\frac{C s \log n}{8}\right),
   \end{equation}
   yielding the bound:
   \begin{equation}
   \|\mathbf{A}_{S_c}^\top \mathbf{w}\|_2 + \|\mathbf{A}_{S_f}^\top \mathbf{w}\|_2 \leq \sqrt{C} \sigma \sqrt{\frac{s\log n}{m}},
   \end{equation}

2. For \( |S_f| = 1, 2, \dots, m - s \), the probability expression becomes:
   \begin{equation}
    1 - 4s \exp\left(-\frac{(\delta - T)^2}{8\sigma^2} m\right) - \exp\left(-\frac{C s \log n}{8}\right) - \frac{1}{n^{C |S_f|/8}},
   \end{equation}
   which is a clearer representation of the same expression. Here, \( |S_f| \) is fixed, and as \( n \) increases, the term \( \frac{1}{n^{C |S_f| / 8}} \) decays rapidly, becoming negligible for large \( n \). The corresponding bound for the sum of the projections remains:
\begin{equation}
\|\mathbf{A}_{S_c}^\top \mathbf{w}\|_2 + \|\mathbf{A}_{S_f}^\top \mathbf{w}\|_2 \leq \sqrt{C} \left( \sqrt{s} + \sqrt{|S_f|} \right) \sigma \sqrt{\frac{\log n}{m}}.
\end{equation}
In conclusion, for any fixed \( |S_f| \), the probability that the bound holds tends to 1 as \(m\) and \( n \) increases.
\end{proof}

Then we have the following proof.
\begin{proof}[Proof of Theorem \ref{thm:fdr_error1}]
The support set \(\hat{S}\) can be partitioned into the correctly identified support \(S_c\) and the falsely identified support \(S_f\), while the true signal set \(S\) can be partitioned into \(S_c\) and the missed signal set \(S_m\). Rearranging the columns in \(A\) and the corresponding element order in \(\mathbf{x}\) and \(\mathbf{\hat{x}}\), we have the block representation as:
\begin{equation}
\mathbf{A}_{\hat{S}} = [\, \mathbf{A}_{S_c} \; \mathbf{A}_{S_f} \,], \quad \mathbf{A}_{S} = [\, \mathbf{A}_{S_c} \; \mathbf{A}_{S_m} \,], \quad
\mathbf{x} = \begin{bmatrix}
\mathbf{x}_c \\
\mathbf{0} \\
\mathbf{x}_m \\
\mathbf{0} \\
\end{bmatrix},\quad
\mathbf{\hat{x}} = \begin{bmatrix}
\hat{\mathbf{x}}_c \\
\hat{\mathbf{x}}_f \\
\mathbf{0} \\
\mathbf{0}
\end{bmatrix}.
\end{equation}
where \( \mathbf{A}_{\hat{S}}, \mathbf{A}_{S_c}, \mathbf{A}_{S_m}\) and \(\mathbf{A}_{S_f}\) are the submatrix of \( \mathbf{A} \) with columns indexed by \( \hat{S}, S_c, S_m,\) and \(S_f\) separately. Besides, we decompose the estimated error as:
\begin{equation}
    \hat{\mathbf{x}} - \mathbf{x} =  \begin{bmatrix}
\hat{\mathbf{x}}_c \\
\hat{\mathbf{x}}_f \\
\mathbf{0} \\
\mathbf{0}
\end{bmatrix} - \begin{bmatrix}
\mathbf{x}_c \\
\mathbf{0} \\
\mathbf{x}_m \\
\mathbf{0} \\
\end{bmatrix}=\begin{bmatrix}
\mathbf{h}_c \\
\mathbf{h}_f \\
\mathbf{h}_m \\
\mathbf{0} \\
\end{bmatrix}
\end{equation}
where we denote \(\mathbf{h}_c = \hat{\mathbf{x}}_c - \mathbf{x}_c\), \(\mathbf{h}_f = \hat{\mathbf{x}}_f - 0\), \(\mathbf{h}_m = 0 - \mathbf{x}_m\). \\
Thus, by triangle inequality,
\begin{equation} \label{eq:triangle}
\|\mathbf{\hat{x}}-\mathbf{x}\|_2 \le \|\mathbf{h}_c\|_2 + \|\mathbf{h}_f\|_2 + \|\mathbf{h}_m\|_2 = \|\mathbf{h}_c\|_2 + \|\mathbf{h}_f\|_2 + \|\mathbf{x}_m\|_2.
\end{equation}
Given a fixed support set \( \hat{S} \), the \TheName{} estimator \( \hat{\mathbf{x}} \) defined in (\ref{eq:estimator}) is obtained by solving:
\begin{equation}
\hat{\mathbf{x}}_{\hat{S}} = \arg\min_{\mathbf{z}_{\hat{S}} \in \mathbb{R}^{|\hat{S}|}} \frac{1}{2} \|\mathbf{y} - \mathbf{A}_{\hat{S}}\mathbf{z}_{\hat{S}}\|_2^2,
\end{equation}
Therefore, the full estimator \( \hat{\mathbf{x}} \in \mathbb{R}^n \) is obtained by setting \( \hat{x}_i = 0 \) for all \( i \notin \hat{S} \).

Since \(\hat{\mathbf{x}}_{\hat{S}}\) minimizes the least squares objective over the support set \(\hat{S}\), it satisfies the first-order optimality condition:
\begin{equation}
\mathbf{A}_{\hat{S}}^\top (\mathbf{A}_{\hat{S}}\hat{\mathbf{x}}_{\hat{S}} - \mathbf{y}) = \mathbf{0}.
\end{equation}
Substituting the observation model \(\mathbf{y} = \mathbf{A}\mathbf{x} + \mathbf{w} = \mathbf{A}_{S_c}\mathbf{x}_{S_c} + \mathbf{A}_{S_f}\mathbf{x}_{S_f} + \mathbf{w}\) and \(A_{\hat{S}}\hat{\mathbf{x}}_{\hat{S}} = \mathbf{A}_{S_c}\mathbf{\hat{x}}_{S_c} + \mathbf{A}_{S_m}\mathbf{\hat{x}}_{S_m}\) into the optimality condition, we get:
\begin{equation}
\mathbf{A}_{\hat{S}}^\top (\left( \mathbf{A}_{S_c}\mathbf{\hat{x}}_{S_c} + \mathbf{A}_{S_f}\mathbf{\hat{x}}_{S_f} \right)- \left( \mathbf{A}_{S_c}\mathbf{x}_{S_c} + \mathbf{A}_{S_m}\mathbf{x}_{S_m} + \mathbf{w}\right)) = \mathbf{0}.
\end{equation}
which leads to:
\begin{equation}
\mathbf{A}_{\hat{S}}^\top (\mathbf{A}_{S_c}\left(\hat{\mathbf{x}}_{S_c} - \mathbf{x}_{S_c}\right) + \mathbf{A}_{S_f}\left( \hat{\mathbf{x}}_{S_f} - 0\right) + \mathbf{A}_{S_m}\left( 0 - \mathbf{x}_{S_m}\right) - \mathbf{w}) = \mathbf{0}.
\end{equation}
Using matrix block decomposition on \(\mathbf{A}_{\hat{S}}^\top\), we obtain the equation corresponding to \(A_{S_c}^\top\):
\begin{equation}
\mathbf{A}_{S_c}^\top (\mathbf{A}_{S_c}\mathbf{h}_{S_c} + \mathbf{A}_{S_f}\mathbf{h}_{S_f} + \mathbf{A}_{S_m}\mathbf{h}_{S_m} - \mathbf{w}) = \mathbf{0}.
\end{equation}
Rearranging the terms gives:
\begin{equation}
\mathbf{A}_{S_c}^\top \mathbf{A}_{S_c}\mathbf{h}_{S_c} = \mathbf{A}_{S_c}^\top \mathbf{w} - \mathbf{A}_{S_c}^\top \mathbf{A}_{S_f}\mathbf{h}_{S_f} -  \mathbf{A}_{S_c}^\top \mathbf{A}_{S_m}\mathbf{h}_{S_m}.
\end{equation}
By triangle inequality, we have:
\begin{equation}
\|\mathbf{A}_{S_c}^\top \mathbf{A}_{S_c}\mathbf{h}_{S_c}\|_2 \le \|\mathbf{A}_{S_c}^\top \mathbf{w}\|_2 + \|\mathbf{A}_{S_c}^\top \mathbf{A}_{S_f}\mathbf{h}_{S_f}\|_2 + \|\mathbf{A}_{S_c}^\top \mathbf{A}_{S_m}\mathbf{h}_{S_m}\|_2.
\end{equation}

Notice that the size of the correctly identified support set satisfies \(|S_c| \leq |\hat{S}| \leq m\). By the RE condition, we have the bound:
\begin{equation}
\kappa_{\min} \|\mathbf{h}_{S_c}\|_2 \leq \|\mathbf{A}_{S_c} \mathbf{h}_{S_c}\|_2 \leq \frac{1}{\kappa_{\min}} \|\mathbf{A}_{S_c}^\top \mathbf{A}_{S_c}\mathbf{h}_{S_c}\|_2.
\end{equation}
Thus,
\begin{equation}
\|\mathbf{h}_{S_c}\|_2 \leq \frac{1}{\kappa^2_{\min}} \|\mathbf{A}_{S_c}^\top \mathbf{A}_{S_c}\mathbf{h}_{S_c}\|_2.
\end{equation}
Substituting this into the norm inequality, we obtain:
\begin{equation}
\|\mathbf{h}_{S_c}\|_2 \leq \frac{1}{\kappa_{\min}^2} \left( \|\mathbf{A}_{S_c}^\top \mathbf{w}\|_2 + \|\mathbf{A}_{S_c}^\top \mathbf{A}_{S_f}\mathbf{h}_{S_f}\|_2 + \|\mathbf{A}_{S_c}^\top \mathbf{A}_{S_m}\mathbf{h}_{S_m}\|_2\right).
\end{equation}
Similarly, since $|S_m|\le S\le m$, by the RE condition,
\begin{equation}
    \|\mathbf{A}_{S_c}^\top \mathbf{A}_{S_m}\mathbf{h}_{S_m}\|_2 \le \kappa_{\max} \|\mathbf{A}_{S_m}\mathbf{h}_{S_m}\|_2 \le \kappa^2_{\max}\|\mathbf{h}_{S_m}\|_2 = \kappa^2_{\max}\|\mathbf{x}_{S_m}\|_2
\end{equation}
Besides, the mutual coherence condition assumption leads to the following bound:
\begin{equation}
\|\mathbf{A}_{S_c}^\top \mathbf{A}_{S_f} \mathbf{h}_{S_f}\|_2 \leq \|\mathbf{A}_{S_c}^\top \mathbf{A}_{S_f}\|_2 \|\mathbf{h}_{S_f}\|_2 \leq \gamma \|\mathbf{h}_{S_f}\|_2.
\end{equation}
Substituting these bounds into the previous inequality gives:
\begin{equation} \label{eq:ineq3}
\|\mathbf{h}_{S_c}\|_2 \leq \frac{1}{\kappa_{\min}^2} \left( \|\mathbf{A}_{S_c}^\top \mathbf{w}\|_2 + \gamma \|\mathbf{h}_{S_f}\|_2 + \kappa^2_{\max}\|\mathbf{x}_{S_m}\|_2\right).
\end{equation}
By similar reasoning, we obtain the bound for $\|\mathbf{h}_{S_f}\|_2$:
\begin{equation} \label{eq:ineq4}
\|\mathbf{h}_{S_f}\|_2 \leq \frac{1}{\kappa_{\min}^2} \left( \|\mathbf{A}_{S_f}^\top \mathbf{w}\|_2 + \gamma \|\mathbf{h}_{S_c}\|_2 + \kappa^2_{\max}\|\mathbf{x}_{S_m}\|_2\right).
\end{equation}
Adding these two inequalities together yields:
\begin{equation}
\|\mathbf{h}_{S_c}\|_2 + \|\mathbf{h}_{S_f}\|_2 \leq \frac{1}{\kappa_{\min}^2} \left[\left( \|\mathbf{A}_{S_c}^\top \mathbf{w}\|_2 + \|\mathbf{A}_{S_f}^\top \mathbf{w}\|_2 \right) + \gamma\left(\|\mathbf{h}_{S_c}\|_2 + \|\mathbf{h}_{S_f}\|_2 \right) + 2\kappa^2_{\max}\|\mathbf{x}_{S_m}\|_2\right].
\end{equation}
Rearranging the terms to isolate \( \|\mathbf{h}_{S_c}\|_2 + \|\mathbf{h}_{S_f}\|_2 \) on one side, we get:
\begin{equation}
\|\mathbf{h}_{S_c}\|_2 + \|\mathbf{h}_{S_f}\|_2 \leq \frac{1}{\kappa_{\min}^2 - \gamma} \left( \|\mathbf{A}_{S_c}^\top \mathbf{w}\|_2 + \|\mathbf{A}_{S_f}^\top \mathbf{w}\|_2 \right) + \frac{2}{\kappa_{\min}^2 - \gamma}\|\mathbf{x}_{S_m}\|_2.
\end{equation}
Recall the triangle inequality in (\ref{eq:triangle}), thus
\begin{equation}\label{eq:upper_bound}
\|\mathbf{\hat{x}}-\mathbf{x}\|_2 \leq \frac{1}{\kappa_{\min}^2 - \gamma} \left( \|\mathbf{A}_{S_c}^\top \mathbf{w}\|_2 + \|\mathbf{A}_{S_f}^\top \mathbf{w}\|_2 \right) + \left(\frac{2}{\kappa_{\min}^2 - \gamma}+1\right)\|\mathbf{x}_{S_m}\|_2.
\end{equation}
By Theorem \ref{thm:SupportRecoveryFDRControl}, the probability for \(S \subseteq \hat{S}\) is:
\begin{equation}
    \mathbb{P}( S \subseteq \hat{S} ) = 1-4s \exp\left(-\frac{(\delta - T)^2}{8\sigma^2} m\right).
\end{equation}
 Under the condition \(S \subseteq \hat{S}\), we have \( \mathbf{x}_{S_m} = \mathbf{0} \). Utilizing Lemma \ref{lemma:subexponentialbound}, this leads to the final bound with probability at least \(1- 4s \exp\left(-\frac{(\delta - T)^2}{8\sigma^2} m\right) - \exp\left(-\frac{C s }{8}\log n\right) - \exp\left( -\frac{C |S_f| }{8}\log n \right)\):
\begin{equation}
\|\mathbf{\hat{x}}-\mathbf{x}\|_2 \leq \frac{\sqrt{C} \sigma \left( \sqrt{s} + \sqrt{|S_f|} \right)} {\left( \kappa_{\min}^2 - \gamma  \right)}\sqrt{\frac{\log n}{m}}
\end{equation}
\end{proof}

\begin{theorem}[Refined Reconstruction Error Bound]
\label{thm:fdr_error3}
Under Assumptions \ref{assumption:A1}-\ref{assumption:A2}, there exists a constant \( C > \frac{2}{\log 2} \) such that, for an FDR control level of \( q = c \cdot \frac{s}{m} \), where \( c \) is a constant ensuring \( q \in (0,1) \), and for any \( \eta \in (0,1) \), the \TheName{} estimator \( \hat{\mathbf{x}} \) satisfies:
\begin{equation}
\|\mathbf{\hat{x}}-\mathbf{x}\|_2\le
\frac{\sqrt{C}}{\kappa_{\min}^2 - \gamma}\left(1+\sqrt{\frac{c}{\eta}}\right) \sigma \sqrt{\frac{s\log n}{m}}
\end{equation}
with probability at least \(1- 4s \exp\left(-\frac{(\delta - T)^2}{8\sigma^2} m\right) - \exp\left(-\frac{C s }{8}\log n\right) - \exp\left( -\frac{C |S_f| }{8}\log n \right)- \eta\).
\end{theorem}

\begin{proof}[Proof of Theorem \ref{thm:fdr_error3}]
Since Theorem \ref{thm:SupportRecoveryFDRControl} ensures that \(\mathbb{E}\left( \frac{|S_f|}{|\hat{S}|} \right) \leq q\), applying Markov's inequality for any \(\eta \in (0,1)\) yields:
\begin{equation}
\mathbb{P}\!\Bigl(\tfrac{|S_f|}{|\hat{S}|}  \ge \frac{q}{\eta}\Bigr)\leq \mathbb{P}\!\Bigl(\tfrac{|S_f|}{|\hat{S}|}  \ge \mathbb{E}(\frac{|S_f|}{|\hat{S}|})\cdot\frac{1}{\eta}\Bigr) \le \eta.
\end{equation}
By Theorem \ref{thm:fdr_error1}, we deduce that with probability at least \(1- 4s \exp\left(-\frac{(\delta - T)^2}{8\sigma^2} m\right) - \exp\left(-\frac{C s }{8}\log n\right) - \exp\left( -\frac{C |S_f| }{8}\log n \right)- \eta\),
\begin{equation}
\|\mathbf{\hat{x}}-\mathbf{x}\|_2 \leq \frac{C \sigma \left( \sqrt{s} + \sqrt{|S_f|} \right)} { \left( \kappa_{\min}^2 - \gamma  \right)}\sqrt{\frac{\log n}{m}}
\end{equation}
\begin{equation}
\leq \frac{\sqrt{C} \sigma \left( \sqrt{s} + \sqrt{|\hat{S}| q} \right)} { \left( \kappa_{\min}^2 - \gamma  \right)}\sqrt{\frac{\log n}{m}}
\end{equation}
\begin{equation}
\le
\frac{\sqrt{C}}{\kappa_{\min}^2 - \gamma} \sigma \sqrt{\frac{s\log n}{m}}
+ \frac{\sqrt{C}}{(\kappa_{\min}^2 - \gamma)\sqrt{\eta}} \sigma \sqrt{\frac{m q \log n}{m}}
\end{equation}
\begin{equation}
\le
\frac{\sqrt{C}}{\kappa_{\min}^2 - \gamma}\left(1+\sqrt{\frac{c}{\eta}}\right) \sigma \sqrt{\frac{s\log n}{m}}
\end{equation}
\end{proof}

\subsection{Expected Reconstruction Error and FDR Level}
This section extends the previous results by focusing on the expected reconstruction error as a function of the target FDR level, providing a precise characterization of the performance of the estimator.

\begin{theorem}[Expected Reconstruction Error and FDR Level]
\label{thm:fdr_error2}
Under the conditions in Assumptions~\ref{assumption:A1}-\ref{assumption:A2}, and provided that (\ref{eq:Int_result}) in Theorem \ref{thm:fdr_error1} holds (with high probability), for $q = O(\frac{s}{m})$, the \TheName{} estimator \(\hat{\mathbf{x}}\) satisfies the following expected reconstruction error bound:
\begin{equation}
\mathbb{E}\bigl[\|\hat{\mathbf{x}} - \mathbf{x}\|_2\bigr] 
\le O\left(\sigma \sqrt{\frac{s \log n}{m}}\right).
\end{equation}
\end{theorem}

\begin{proof}[Proof of Theorem \ref{thm:fdr_error2}]
Starting from (\ref{eq:Int_result}), we can derive the following error bound: 
\begin{equation} \label{eq:errorBound}
\|\hat{\mathbf{x}} - \mathbf{x}\|_2 
\leq \frac{\sqrt{C} \sigma \left( \sqrt{s} + \sqrt{|S_f|} \right)} {\left( \kappa_{\min}^2 - \gamma  \right)}\sqrt{\frac{\log n}{m}}.
\end{equation}
\begin{equation}
= \frac{\sqrt{C} \sigma } {\left( \kappa_{\min}^2 - \gamma  \right)} \sqrt{\frac{s\log n}{m}}
+\frac{\sqrt{C} \sigma } {\left( \kappa_{\min}^2 - \gamma  \right)} \sqrt{\frac{\frac{|S_f|}{|\hat{S}|} |\hat{S}| \log n}{m}}.
\end{equation}
noting that \(|\hat{S}| \le m\) (see Assumption \ref{assumption:A1}),
\begin{equation}
\le \frac{\sqrt{C} \sigma } {\left( \kappa_{\min}^2 - \gamma  \right)} \sqrt{\frac{s\log n}{m}} + \frac{\sqrt{C} \sigma } {\left( \kappa_{\min}^2 - \gamma  \right)} \sqrt{\frac{\frac{|S_f|}{|\hat{S}|} m \log n}{m}},
\end{equation}
\begin{equation}
= \frac{\sqrt{C} \sigma } {\left( \kappa_{\min}^2 - \gamma  \right)} \sqrt{\frac{s \log n}{m}}
+ \frac{\sqrt{C} \sigma } {\left( \kappa_{\min}^2 - \gamma  \right)} \sqrt{\frac{|S_f|}{|\hat{S}|} \log n}.
\end{equation}
Taking expectations and applying Jensen's inequality to the second term:
\begin{equation}
\mathbb{E}\|\hat{\mathbf{x}} - \mathbf{x}\|_2  \le \frac{\sqrt{C} \sigma } {\left( \kappa_{\min}^2 - \gamma  \right)} \sqrt{\frac{s \log n}{m}}
+ \frac{\sqrt{C} \sigma } {\left( \kappa_{\min}^2 - \gamma  \right)} \mathbb{E}\left(\sqrt{\frac{|S_f|}{|\hat{S}|} \log n}\right).
\end{equation}
\begin{equation}
\leq \frac{\sqrt{C} \sigma } {\left( \kappa_{\min}^2 - \gamma  \right)} \sqrt{\frac{s \log n}{m}}
+\frac{\sqrt{C} \sigma } {\left( \kappa_{\min}^2 - \gamma  \right)} \sqrt{\mathbb{E}\Bigl(\frac{|S_f|}{|\hat{S}|}\Bigr) \log n}.
\end{equation}

Substituting \(\mathbb{E}\Bigl(\frac{|S_f|}{|\hat{S}|}\Bigr) \le q\) and assuming \(q = O\Bigl(s/m\Bigr)\), we conclude:
\begin{equation}\label{eq:part2}
\mathbb{E}\bigl[\|\hat{\mathbf{x}} - \mathbf{x}\|_2\bigr] \le O\left(\sigma \sqrt{\frac{s \log n}{m}}\right).
\end{equation}
\end{proof}

\section{Experiments}
This section evaluates the performance of \TheName{} through both synthetic simulations and real-world applications, examining its accuracy in signal reconstruction and practical utility across domains.
\subsection{Simulation Experiments}
This section presents controlled simulations to evaluate the effectiveness of the proposed \TheName{} framework in support recovery and signal reconstruction.
\subsubsection{Experimental Setup}

We perform a comprehensive set of controlled simulations to evaluate the performance of \TheName{} in recovering sparse signals from compressed linear measurements. The experiments assess both support identification accuracy and signal reconstruction quality under varying sparsity levels, measurement sizes, and noise conditions. Our method is benchmarked against two widely used baselines: LASSO-guided and Orthogonal Matching Pursuit (OMP)-guided compressive sensing. The evaluation focuses on false discovery rate (FDR) control, detection power, and reconstruction accuracy.

\paragraph{Synthetic Dataset Generation}

Synthetic datasets are generated according to the standard compressive sensing model \( \mathbf{y} = \mathbf{A}\mathbf{x} + \mathbf{w} \), where \( \mathbf{x} \in \mathbb{R}^n \) is a sparse signal and \( \mathbf{A} \in \mathbb{R}^{m \times n} \) is the measurement matrix. We vary the signal dimension \( n \in \{500, 1000\} \), number of measurements \( m \in \{50, 100, 200\} \), sparsity level \( s \in \{5, 10\} \), and signal-to-noise ratio (SNR) \( \in \{2, 10, 30, 50\} \) dB. The sparse signal \( \mathbf{x} \) is constructed with \( s \) nonzero entries drawn independently from \( \mathcal{N}(0, 1) \), and the remaining entries are set to zero. The signal is then scaled to unit \( \ell_2 \)-norm. Gaussian noise \( \mathbf{w} \sim \mathcal{N}(0, \sigma^2 \mathbf{I}) \) is added to achieve the desired SNR, with noise level computed via:
\[
\sigma = \sqrt{ \frac{\| \mathbf{A}\mathbf{x} \|_2^2}{10^{\text{SNR}/10}} }.
\]
For each configuration, we generate 20 independent trials with newly sampled signals and measurement matrices. A fixed random seed is used to ensure reproducibility.

\paragraph{Signal and Measurement Matrix Construction}

The measurement matrix \( \mathbf{A} \in \mathbb{R}^{m \times n} \) is generated with a block-diagonal correlation structure. Specifically, the matrix is partitioned into non-overlapping blocks of size 5, with intra-block entries sampled from a multivariate normal distribution with mean zero and pairwise correlation coefficient of 0.6. Each column of \( \mathbf{A} \) is normalized to unit \( \ell_2 \)-norm. This structured design introduces controlled correlation between features, simulating realistic measurement dependencies. The knockoff selection ratio is fixed at 0.1 for all simulations.

\paragraph{Performance Metrics and Implementation Details}

Support recovery is evaluated using three standard metrics: false discovery rate (FDR), defined as the proportion of false positives among selected indices; power (true positive rate), which measures the proportion of correctly identified nonzero entries; and the F1-score, the harmonic mean of precision and recall. Reconstruction performance is assessed via relative error \( \|\hat{\mathbf{x}} - \mathbf{x}\|_2 / \|\mathbf{x}\|_2 \), and measurement reconstruction error \( \|\mathbf{A}\hat{\mathbf{x}} - \mathbf{y}\|_2 \).

The Knockoff-based selection in \TheName{} is implemented using the \texttt{knockpy} library, adopting a model-X knockoff construction for computational convenience. The feature importance statistic is defined as the difference in cross-validated Lasso coefficients between original features and their corresponding knockoffs. The target FDR level \( q \) is set to 0.1 throughout all experiments.

Baseline methods (LASSO and OMP) are implemented using \texttt{scikit-learn}. For LASSO, we report results using a fixed regularization parameter \( \lambda = 0.1 \) across all settings.

All simulations are repeated 20 times per configuration to account for statistical variation, and result visualizations in subsequent sections summarize the aggregate trends across these trials. To support reproducibility, the full implementation of the simulation experiments, including data generation, preprocessing, and evaluation code, is publicly available online\footnote{\url{https://github.com/xiaochenzhang166/KnockoffCS/tree/main}}.

\subsubsection{Results of Simulation Experiments}

We report experimental results across various noise and sampling regimes to evaluate both the support recovery accuracy and reconstruction quality of the proposed method. Figures~\ref{fig:f1_score_comparison} through~\ref{fig:reconstruction_error_comparison} present quantitative comparisons among \TheName{}, LASSO, and OMP-guided compressive sensing. Specifically, Figures~\ref{fig:f1_score_comparison} and~\ref{fig:fdr_power_tradeoff} examine support recovery via F1-score and the trade-off between false discovery rate (FDR) and statistical power. Figures~\ref{fig:relative_error_comparison} and~\ref{fig:reconstruction_error_comparison} focus on relative and absolute reconstruction errors, respectively.

\begin{figure}[htbp]
  \centering
  \includegraphics[width=1\linewidth]{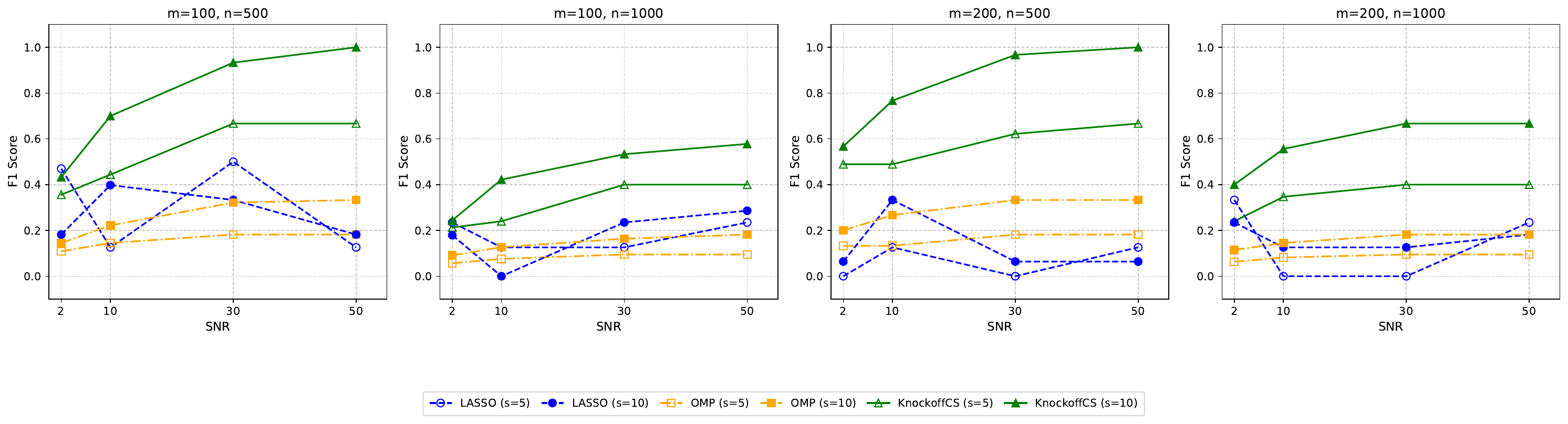}
  \caption{F1-score Comparison Across SNR Levels for Compressive Sensing Methods}
  \label{fig:f1_score_comparison}
\end{figure}

\begin{figure}[htbp]
  \centering
  \includegraphics[width=1\linewidth]{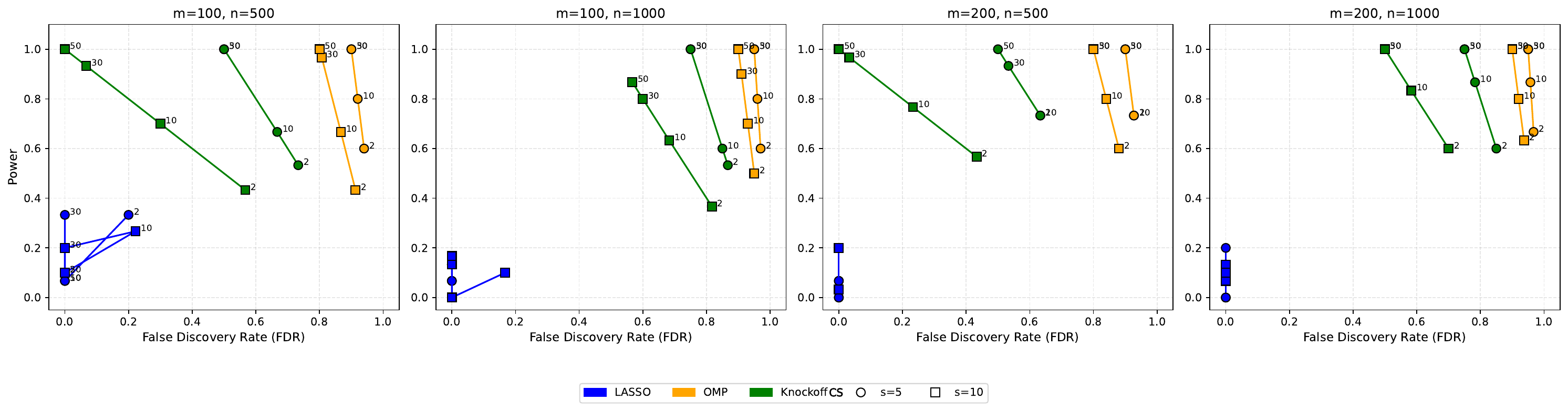}
  \caption{FDR–power trade-off of compressive sensing methods ($m=100$, $n=500$). Methods closer to the top-left corner indicate better performance, with higher power and lower FDR. Signal-to-noise ratios (SNRs) are annotated near each point (Note: For clarity, overlapping points are labeled only once, and SNR labels are omitted for LASSO in the last three subplots due to frequent overlaps.)}
  \label{fig:fdr_power_tradeoff}
\end{figure}

\begin{figure}[htbp]
  \centering
  \includegraphics[width=1\linewidth]{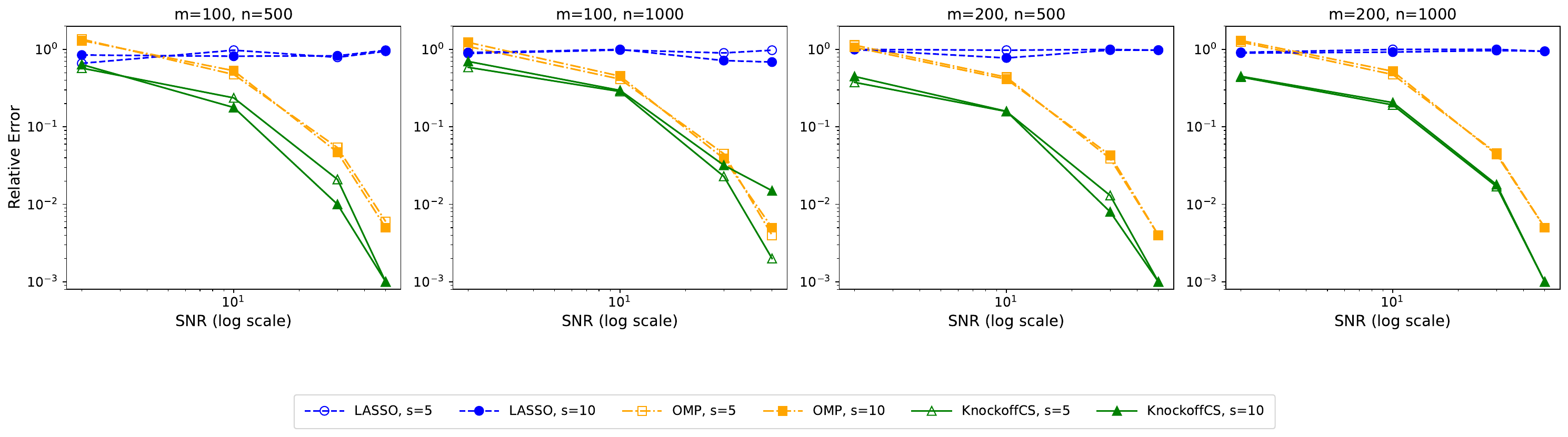}
  \caption{Relative Error Comparison Across SNR Levels for Compressive Sensing Methods}
  \label{fig:relative_error_comparison}
\end{figure}

\begin{figure}[htbp]
  \centering
  \includegraphics[width=1\linewidth]{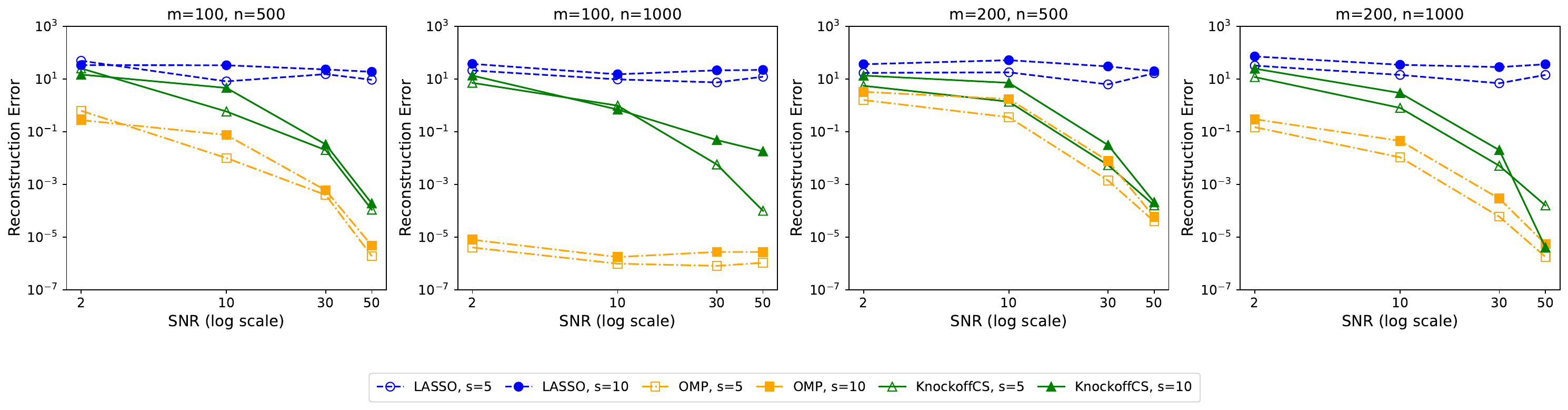}
  \caption{Reconstruction Error Comparison Across SNR Levels for Compressive Sensing Methods}
  \label{fig:reconstruction_error_comparison}
\end{figure}

\paragraph{Support Recovery Performance}

Our method consistently outperforms baseline approaches in identifying the true support set across all tested configurations. As illustrated in Figure~\ref{fig:f1_score_comparison}, its average F1-score is approximately 3.9× higher than that of LASSO and 3.45× higher than that of OMP. For instance, under high SNR (50 dB) with $m=200$ and $n=1000$, it achieves an F1-score that is 3.66× higher than both baselines. This performance advantage remains stable across different values of $m$ and $n$, indicating robustness to variations in measurement density and signal dimension.

Figure~\ref{fig:fdr_power_tradeoff} further demonstrates the superior balance between discovery precision and detection power achieved by our method. In a representative setting with $m=100$ and $n=500$, all methods perform poorly at low SNR (2 dB); however, the proposed method exhibits rapid improvement as SNR increases. At 50 dB, it achieves near-perfect power while maintaining an FDR below 0.02. In contrast, OMP exhibits persistently high FDRs, often exceeding 0.8 even at high SNR levels, indicating substantial over-selection and lack of effective support discrimination. Although all methods experience degraded performance under lower sparsity ($s=5$), our method remains consistently closest to the ideal top-left region in the FDR–power plane.

\paragraph{Signal Reconstruction Quality}

We next evaluate the accuracy of the reconstructed signal in both the signal and measurement domains. As shown in Figure~\ref{fig:relative_error_comparison}, the proposed approach achieves relative reconstruction errors as low as $10^{-3}$ under high SNR conditions, significantly outperforming LASSO and OMP, with errors typically ranging between $10^{-2}$ and $10^{-1}$. This reflects a more accurate estimation of the signal magnitude and structure. Figure~\ref{fig:reconstruction_error_comparison} presents the corresponding measurement reconstruction errors, where our method again demonstrates a clear advantage, particularly in the mid- to high-SNR regimes. These improvements confirm that accurate support identification translates into better downstream signal recovery.

\paragraph{Summary}

Overall, the simulation results establish the proposed framework as a robust and accurate compressive sensing solution. Its FDR-controlled selection process enables precise support recovery, which in turn yields superior signal reconstruction performance. Unlike traditional greedy or regularized approaches that are sensitive to noise and hyperparameter tuning, our method maintains consistently high performance across a wide range of configurations, making it particularly suitable for settings with strong noise, correlation, or measurement constraints.

\subsection{Machine Learning with Recovered Signals from Noisy, Low-dimensional Observations}
This section evaluates the practical efficacy of \TheName{} in a challenging, realistic scenario: performing machine learning tasks using signals recovered from noisy, low-dimensional observations. These observations are derived from a diverse array of real-world datasets, which are intentionally subjected to noise augmentation and significant dimensionality reduction (compression). This process simulates demanding real-world conditions where only noisy and incomplete measurements are typically available. The primary objective is to demonstrate that signals reconstructed by \TheName{} from these degraded observations—owing to its statistically principled support recovery—provide a superior basis for downstream prediction tasks compared to signals recovered by established compressive sensing techniques.

\begin{table}[ht]
  \centering
  \caption{Overview of datasets used in model evaluation.}
  \label{tab:datasets}
  \footnotesize
  \begin{tabular}{lcccccc}
  \toprule
  \textbf{Dataset} & \textbf{$d_{\text{orig}}$} & \textbf{$d_{\text{noise}}$} & \textbf{$m$(compressed)} & \textbf{Split} & \textbf{Domain} & \textbf{Type} \\
  \midrule
  Infrared Thermography  & \multirow{2}{*}{25} & \multirow{2}{*}{475} & \multirow{2}{*}{100} & \multirow{2}{*}{3500/831} & \multirow{2}{*}{Medical imaging} & \multirow{2}{*}{regression} \\
  Temperature~\cite{facial} & & & & & \\
  \hline
  Mars Asteroid & \multirow{2}{*}{24} & \multirow{2}{*}{476} & \multirow{2}{*}{100} & \multirow{2}{*}{2500/500} &  \multirow{2}{*}{Astronomy} & \multirow{2}{*}{regression}\\
  Observation~\cite{mars} & & & & &\\
  \hline
  \multirow{2}{*}{Million Song~\cite{million_song_2011}} & \multirow{2}{*}{90} & \multirow{2}{*}{1910} & \multirow{2}{*}{500} & \multirow{2}{*}{800/200} & \multirow{2}{*}{Music analysis} & \multirow{2}{*}{regression} \\
   & & & & &\\
  \hline
  Nanofluid Density & \multirow{2}{*}{10} & \multirow{2}{*}{990} & \multirow{2}{*}{100} & \multirow{2}{*}{800/200} &  \multirow{2}{*}{Materials Science} & \multirow{2}{*}{regression}\\
  Prediction~\cite{nanofluid} & & & & &\\
  \hline
  \multirow{2}{*}{HIGGS~\cite{higgs_dataset}} & \multirow{2}{*}{24} & \multirow{2}{*}{476} & \multirow{2}{*}{100} & \multirow{2}{*}{2500/500} & \multirow{2}{*}{Physics} & Classification\\
  & & & & & & (2 classes)\\
  \hline
  \multirow{2}{*}{Kepler Exoplanet~\cite{kepler}} & \multirow{2}{*}{9} & \multirow{2}{*}{491} & \multirow{2}{*}{100} & \multirow{2}{*}{2500/500} & \multirow{2}{*}{Astronomy} & Classification\\
  & & & & && (3 classes)\\
  \hline
  MAGIC Gamma & \multirow{2}{*}{10} & \multirow{2}{*}{990} & \multirow{2}{*}{100} & \multirow{2}{*}{800/200} & \multirow{2}{*}{Astronomy} & Classification\\
  Telescope~\cite{telescope} & & & & & & (2 classes)\\
  \hline
  Yahoo! Learning to & \multirow{2}{*}{404} & \multirow{2}{*}{2596} & \multirow{2}{*}{800} & \multirow{2}{*}{800/200} & \multirow{2}{*}{Web search} &  Classification \\
  Rank Challenge~\cite{yahoo} & & & & & & (5 classes)\\
  \bottomrule
  \end{tabular}
\end{table}

To comprehensively assess the practical effectiveness of \TheName{}, we designed a robust experimental framework. This framework encompasses diverse real-world prediction tasks, including both regression and classification settings, and is structured to clearly demonstrate the advantages of \TheName{}. This subsection details the dataset selection, the process for synthesizing noisy and low-dimensional observations, and the methodology for signal reconstruction and subsequent evaluation of downstream predictive performance.

\paragraph{Datasets and Domains}
Our empirical validation leverages eight real-world datasets (Table~\ref{tab:datasets}), carefully selected to represent a wide spectrum of application domains: medical imaging, astrophysics, music analysis, materials science, web search, and particle physics. This diversity is crucial for evaluating the generalizability and robustness of \TheName{}. The core objective is to demonstrate that sparse signals recovered by our statistically-grounded Knockoff-guided method provide a more reliable basis for downstream prediction tasks compared to signals reconstructed using established baselines: LASSO, OMP, CLIME~\cite{CLIME}, and Dantzig Selector~\cite{candes2007dantzig}. The dataset selection ensures a challenging mix of regression tasks, binary classification, and more complex multi-class classification problems (e.g., Yahoo! Learning to Rank Challenge and Kepler Exoplanet datasets). For rigorous and fair comparison, all methods are evaluated on identical data splits (typically 80\% train/20\% test, randomly partitioned without replacement).

\paragraph{Observation Synthesis and Signal Reconstruction}
To simulate the challenging conditions often encountered in practice, we first synthesize a high-dimensional and noisy environment from each real-world dataset. This involves augmenting each original dataset by appending $d_{\text{noise}}$ synthetic Gaussian noise features, drawn independently from $\mathcal{N}(0, \sigma^2)$ (where $\sigma$ is 0.5 times the mean standard deviation of the original features). This step significantly increases the dimensionality of the signal to $n = d_{\text{orig}} + d_{\text{noise}}$ and introduces a substantial number of irrelevant variables, thereby creating a stringent test for the ability of \TheName{} to accurately identify the true signal support (the original features) while controlling false discoveries among the noise. To prevent any method from exploiting trivial feature ordering and to simulate realistic conditions where informative variable positions are unknown, the augmented high-dimensional signal vector $\mathbf{x}$ is randomly permuted.

Subsequently, noisy, low-dimensional observations $\mathbf{y}$ are synthesized by applying a common measurement matrix $\mathbf{A} \in \mathbb{R}^{m \times n}$ (with entries i.i.d. from $\mathcal{N}(0, 1/m)$) to the permuted high-dimensional signal $\mathbf{x}$, i.e., $\mathbf{y} = \mathbf{A} \mathbf{x}$. Our experiments focus on the following setting for signal recovery:
\begin{itemize}
\item \textbf{Recovery with the known measurement matrix $\mathbf{A}$}: For each dataset, we fix a measurement matrix $\mathbf{A}$, ensuring all compressive sensing methods operate under identical conditions when acquiring the low-dimensional observations. During reconstruction, each method attempts to recover an estimate $\hat{\mathbf{x}}$ of the (permuted) augmented sparse signal from the observations $\mathbf{y}$ using the shared measurement matrix.
\end{itemize}

For methods reliant on explicit support estimation (OMP and our method), the signal estimate is computed via least squares restricted to the recovered support. A key aspect of the evaluation of \TheName{} is its principled approach to support selection: we select the top 1\% of features ranked by their knockoff statistics from the reconstructed $\hat{\mathbf{x}}$. This ensures a consistent, interpretable, and stringent sparsity level across datasets, directly reflecting the mechanism of our method for prioritizing statistically significant features from the noisy, high-dimensional space.

\paragraph{Prediction Models and Evaluation Metrics}
The quality of the reconstructed signals $\hat{\mathbf{x}}$ from each CS method is ultimately judged by their utility in downstream prediction tasks. To this end, a broad suite of prediction models is employed, tailored to the nature of each task. For regression, we use Linear Regression, Lasso, Ridge Regression, Bayesian Ridge Regression, Support Vector Regression (SVR), Random Forest (RF), and XGBoost (XGB). For classification, the models include K-Nearest Neighbors (KNN), Logistic Regression, Gaussian Naive Bayes (GNB), Support Vector Machine (SVM), Random Forest (RF), Multilayer Perceptron Classifier (MLPC), and Extreme Gradient Boosting (XGBoost, XGB). This wide range of models, implemented using \texttt{scikit-learn} and \texttt{xgboost} with default hyperparameters, ensures that the benefits of \TheName{} are assessed across various learning paradigms. Performance is measured using standard metrics: mean squared error (MSE) for regression and accuracy for classification, evaluated on the held-out test set. 

To contextualize the effectiveness of recovery methods, we additionally evaluate two non-reconstruction baselines: (1) using the raw compressed observations directly for prediction, and (2) using the original uncompressed signals. This comparative setup enables us to quantify both the predictive loss due to compression and the extent to which different methods can close that gap. The complete implementation is publicly available\footnote{https://github.com/xiaochenzhang166/KnockoffCS/tree/main} for reproducibility.

\begin{table}
  \centering
  \caption{Comparison of predictive performance in downstream regression task using signals recovered by different compressive sensing methods with the known measurement matrix (values at the \textbf{first places} and the \underline{second places} amoung compressive sensing methods). The \textbf{Compressed} column reports performance using raw compressed observations without recovery, while the \textbf{Original} column shows performance using the original uncompressed signals.}\label{tbl:CS_regression}
  \footnotesize
  \begin{tabular}{@{}c|ccccc|cc@{}}
    \toprule
     method & \TheName{} & \textbf{lasso} & \textbf{OMP} & \textbf{CLIME} & \textbf{Dantzig selector} & \textbf{Compressed} & \textbf{Original} \\
    \midrule
    \multicolumn{8}{c}{MSE on Infrared Thermography Temperature Dataset}\\ \midrule
    BayesianRidge & $\underline{2838.257}$ & $\mathbf{2769.503}$ & 3831.489 & 3173.234 & 3168.624 & 2974.272 & 924.076 \\
    Lasso & $\underline{3008.863}$ & $\mathbf{2985.077}$ & 4251.356 & 3218.709 & 3200.671 & 2991.675 & 1025.365 \\
    Linear & $\mathbf{2837.240}$ & 27915.435 & $\underline{3165.756}$ & 108850.513 & 8850.257 & 3006.343 & 925.583\\
    RF & $\mathbf{3089.573}$ & $\underline{3816.372}$ & 5472.500 & 4360.481 & 4214.778 &  3426.945 & 925.583  \\
    Ridge & 2891.314 & $\underline{2889.420}$ & 3250.001 & 2909.050 & $\mathbf{2854.285}$  & 3006.340 & 924.075\\
    SVR & $\underline{5299.848}$ & $\mathbf{5268.483}$ & 5310.523 & 5757.232 & 5774.931  & 5163.419 & 4228.817\\
    XGB &  $\mathbf{3464.333}$ & 4148.323 & 6127.383 & 4420.737 & $\underline{4046.218}$ & 3762.099 & 914.254 \\
    \midrule
    \multicolumn{8}{c}{MSE on Mars Asteroid Observation Dataset ($\times10^{6}$)}\\ \midrule
    BayesianRidge & 9.428 & $\mathbf{8.579}$ & 9.600 & $\underline{9.265}$ & $\underline{9.265}$ & 9.260 & 9.162 \\
    Lasso & $\mathbf{11.626}$ & 12.007 & 12.374 & $\underline{11.785}$ & $\underline{11.785}$ & 9.771 &  9.201 \\
    Linear & $\mathbf{11.827}$ & 12.023 & 12.376 & $\underline{11.786}$ & $\underline{11.786}$ & 9.783 &  9.413 \\
    RF & 9.857 & $\mathbf{9.060}$ & 10.297 & $\underline{9.473}$ & $\underline{9.473}$ & 9.594 &   9.477 \\
    Ridge & $\mathbf{11.630}$ & 12.023 & 12.376 & $\underline{11.786}$ & $\underline{11.786}$ & 9.783 &  9.193 \\
    SVR & $\underline{8.583}$ & $\mathbf{8.538}$ & 9.564 & 9.292 & 9.292 & 9.303 & 9.148 \\
    XGB & $\mathbf{9.934}$ & $\underline{10.250}$ & 11.502 & 11.131 & 11.131 & 11.395 & 10.837 \\
    \midrule
    \multicolumn{8}{c}{MSE on Million Song Dataset}\\ \midrule
    BayesianRidge & 236.896 & 310.870 & $\mathbf{212.671}$ & $\underline{220.653}$ & 258.239 & 127.193 &  112.655\\
    Lasso & $\mathbf{242.865}$ & 325.949 & 263.020 & $\underline{250.513}$ & 253.261 & 332.422 & 106.259 \\
    Linear & $\underline{224.801}$ & 307.082 & 729.552 & $\mathbf{220.653}$ & 257.837 & 334.733 & 106.678 \\
    RF & $\underline{107.418}$ & 108.364 & $\mathbf{106.758}$ & 147.030 & 143.723 & 131.121 &  94.911 \\
    Ridge & $\underline{192.727}$ & 311.854 & $\mathbf{174.799}$ & 193.918 & 246.654 & 334.733 &  106.677 \\
    SVR & $\mathbf{106.526}$ & $\underline{109.067}$ & 110.583 & 162.388 & 163.020 & 132.135 &  141.302 \\
    XGB &  $\mathbf{118.558}$ & 134.744 & $\underline{125.992}$ & 165.279 & 139.771 & 153.058 & 98.159  \\
    \midrule
    \multicolumn{8}{c}{MSE on Nanofluid Density Prediction Dataset}\\ \midrule
    BayesianRidge & $\mathbf{333.732}$ & $\underline{335.056}$ & 358.927 & 431.002 & 465.867 & 602.405 & 43.116 \\
    Lasso & $\underline{328.165}$ & $\mathbf{320.834}$ & 412.115 & 721.778 & 452.666 & 596.149 &  43.787 \\
    Linear & $\underline{332.696}$ & $\mathbf{325.761}$ & 361.673 & 739.855 & 484.902 & 596.182 & 43.227 \\
    RF & 351.855 & $\mathbf{292.770}$ & $\underline{295.552}$ & 468.142 & 386.882 & 474.506 & 5.011 \\
    Ridge & $\underline{328.126}$ & $\mathbf{325.760}$ & 412.706 & 727.064 & 484.893 & 596.182 & 43.096 \\
    SVR & $\mathbf{280.760}$ & $\underline{282.845}$ & 295.813 & 440.173 & 437.418 & 750.523 & 431.780 \\
    XGB & 428.529 & $\underline{361.049}$ & $\mathbf{341.984}$ & 597.856 & 427.205 & 530.909 & 2.124 \\
    \bottomrule
  \end{tabular}
\end{table}


\begin{table}
  \centering
  \caption{Comparison of predictive performance in downstream classification task using signals recovered by different compressive sensing methods with the known measurement matrix (values at the \textbf{first places} and the \underline{second places} amoung compressive sensing methods). The \textbf{Compressed} column reports performance using raw compressed observations without recovery, while the \textbf{Original} column shows performance using the original uncompressed signals.}\label{tbl:CS_classification}
  \footnotesize
  \begin{tabular}{@{}c|ccccc|cc@{}}
    \toprule
     method & \TheName{} & \textbf{lasso} & \textbf{OMP} & \textbf{CLIME} & \textbf{Dantzig selector} & \textbf{Compressed} & \textbf{Original}\\
    \midrule
    \multicolumn{8}{c}{Accuracy on HIGGS Dataset (2-class classification)}\\ \midrule
    GNB & $\mathbf{0.506}$ & $\underline{0.502}$ & 0.498 & 0.500 & 0.500 & 0.546 & 0.596\\
    KNN & $\mathbf{0.532}$ & $\underline{0.522}$ & 0.520 & 0.518 & 0.518 & 0.514 & 0.510\\
    log & $\underline{0.524}$ & 0.522 & 0.516 & $\mathbf{0.530}$ & 0.522 & 0.524 & 0.644\\
    MLPC & $\underline{0.536}$ & 0.510 & 0.506 & $\mathbf{0.556}$ & 0.532 & 0.524 & 0.614\\
    RF & 0.526 & $\underline{0.530}$ & 0.496 & $\mathbf{0.550}$ & 0.522 & 0.528 & 0.644 \\
    SVM & 0.524 & 0.520 & 0.514 & $\underline{0.528}$ & $\mathbf{0.544}$ & 0.544 & 0.630\\
    XGB & 0.524 & 0.520 & $\mathbf{0.546}$ & $\underline{0.534}$ & 0.464 & 0.572 & 0.648\\
    \midrule
    \multicolumn{8}{c}{Accuracy on Kepler Exoplanet Dataset (3-class classification)}\\ \midrule
    GNB & $\mathbf{0.396}$ & 0.292 & 0.316 & $\underline{0.338}$ & 0.252 & 0.236& 0.406\\
    KNN & $\mathbf{0.470}$ & 0.392 & 0.372 & 0.380 & $\underline{0.446}$ & 0.430 & 0.518\\
    log & $\underline{0.420}$ & 0.402 & 0.404 & 0.408 & $\mathbf{0.428}$ & 0.404 & 0.758\\
    MLPC & $\mathbf{0.400}$ & 0.366 & 0.356 & 0.380 & $\underline{0.390}$ & 0.354 & 0.542\\
    RF & 0.398 & 0.448 & 0.440 & $\mathbf{0.492}$ & $\underline{0.482}$ & 0.502 & 0.808\\
    SVM & 0.460 & 0.462 & 0.458 & $\mathbf{0.492}$ & $\underline{0.488}$ & 0.516 & 0.468\\
    XGB & 0.434 & 0.402 & 0.408 & $\underline{0.438}$ & $\mathbf{0.450}$ & 0.474 & 0.804\\
    \midrule
    \multicolumn{8}{c}{Accuracy on MAGIC Gamma Telescope Dataset (2-class classification)}\\ \midrule
    GNB & 0.460 & 0.565 & $\underline{0.570}$ & 0.545 & $\mathbf{0.575}$ & 0.640 & 0.730\\
    KNN & $\mathbf{0.645}$ & 0.625 & 0.575 & $\underline{0.640}$ & 0.605 & 0.655 & 0.700 \\
    log & $\mathbf{0.655}$ & 0.570  & $\underline{0.630}$  &  0.605  &  0.615 & 0.615 & 0.745\\
    MLPC & $\underline{0.580}$ & 0.545 &  0.555 &  $\mathbf{0.595}$ &  0.545 & 0.625 & 0.720\\
    RF & $\mathbf{0.685}$ & $\underline{0.680}$  & 0.660  &  0.645 & 0.670 & 0.670 & 0.775\\
    SVM & $\underline{0.685}$ & $\mathbf{0.695}$  & $\underline{0.685}$  & 0.660  & 0.655 & 0.685 & 0.755\\
    XGB &  $\mathbf{0.655}$ & $\underline{0.640}$  & 0.605  & 0.630  &  $\mathbf{0.655}$ & 0.625 & 0.810\\
    \midrule
    \multicolumn{8}{c}{Accuracy on Yahoo! Learning to Rank Challenge Dataset (5-class classification)}\\ \midrule
    GNB & 0.390 & $\mathbf{0.470}$ & 0.405 & 0.420 & $\underline{0.460}$ & 0.485 & 0.425\\
    KNN & $\mathbf{0.430}$ & 0.420 & 0.265 & 0.395 & $\underline{0.425}$ & 0.510 & 0.510\\
    log & 0.460 & 0.435 & $\mathbf{0.485}$ & 0.440 & $\underline{0.465}$ & 0.415 & 0.545\\
    MLPC &  0.430 & 0.430 & 0.445 & $\underline{0.450}$ & $\mathbf{0.490}$ & 0.460 & 0.540\\
    RF & $\mathbf{0.530}$ & 0.475 & 0.475 & $\underline{0.495}$ & 0.465 & 0.550 & 0.560\\
    SVM & $\underline{0.500}$ & $\mathbf{0.505}$ & 0.495 & 0.470 & 0.495 & 0.565 & 0.550\\
    XGB &  $\mathbf{0.535}$ & 0.475 & 0.465 & $\underline{0.510}$ & 0.470 &  0.480 & 0.575\\
    \bottomrule
  \end{tabular}
\end{table}

\subsubsection{Experiment Results}
This subsection presents the empirical results and subsequent analysis based on the aforementioned datasets and experimental settings. The focus is on the downstream predictive performance achieved using signals recovered by \TheName{} and baseline methods.

\paragraph{Results on Regression Tasks}
In downstream regression tasks (Table~\ref{tbl:CS_regression}), signals reconstructed by our method consistently facilitated strong predictive outcomes. Across 28 distinct model–dataset combinations, the proposed method secured a top-two ranking among the compressive sensing methods in 82.1\% of instances, highlighting its superior overall consistency compared to LASSO, CLIME, OMP, and Dantzig Selector. For example, on the Infrared Thermography dataset, it achieved top-two performance in six out of seven models. Across all datasets, our method consistently achieved lower MSEs than compressed observations and outperformed signals reconstructed by alternative CS methods, although a gap to the uncompressed benchmark remains in most cases. On more challenging datasets like Nanofluid Density, it maintained strong performance while avoiding the extreme degradation observed in certain baselines, underscoring its robustness under noisy and ill-conditioned scenarios.

\paragraph{Results on Classification Tasks}
Our method also demonstrated robust and competitive performance in classification tasks (Table~\ref{tbl:CS_classification}). Across 28 classifier–dataset combinations, it ranked within the top two CS baselines in 64.3\% of cases, indicating favorable stability. On the MAGIC Gamma dataset, for instance, the proposed method achieved the highest accuracy among CS methods in over half of the classifiers, outperforming compressed signal baselines and narrowing the gap to the uncompressed reference. On the Yahoo! Learning to Rank dataset, it consistently outperformed compressed observations and ranked near the top among CS methods, especially in tree-based models such as RF and XGBoost. While performance on other datasets like Kepler and HIGGS varied depending on classifier type, our method generally exhibited strong generalization capabilities and improved predictive quality relative to simple compression alone.

\paragraph{Summary and Observations}
Across 56 model–dataset pairs, our method consistently demonstrated superior utility for downstream tasks, ranking top-two among CS methods in 71.4\% of cases. Its strong empirical performance, particularly in high-dimensional settings such as Million Song and Yahoo, supports the theoretical foundation of FDR-controlled support recovery. By incorporating compressed and original signal baselines in our evaluation, we show that while uncompressed signals typically yield the best accuracy, our method substantially closes the performance gap. These results underscore that the statistical rigor of the proposed approach in feature selection not only aids compression but also helps refine noisy signals, confirming its practical value and superiority for prediction-driven CS applications.

\section{Discussion and Conclusions}
We present a novel compressive sensing strategy, \TheName{}, that leverages the statistical rigor of Knockoff filters to guide support recovery, thereby improving signal reconstruction quality. By explicitly separating support identification from reconstruction and applying controlled false discovery mechanisms, our method enhances reconstruction robustness, particularly in regimes where conventional techniques struggle. Theoretical analysis demonstrates that principled FDR control ensures reliable recovery under weaker assumptions than conventional $\ell_1$-based methods. Extensive simulation studies confirm that our method achieves substantial gains in F1-score for detecting non-zero components in signals, and consistently reduces reconstruction error compared to standard baselines. Evaluations on real-world datasets further show that the signal reconstructed from our framework achieves strong predictive performance across both regression and classification downstream tasks, frequently matching or even surpassing the performance of models trained on uncompressed signals. These results highlight the proposed method as a robust and principled alternative in compressive sensing, combining rigorous guarantees with strong empirical results.

\section*{Declaration}
\begin{itemize}
\item Funding -  Not applicable

\item Conflicts of interest/Competing interests - Not applicable

\item Ethics approval - No data have been fabricated or manipulated to support your conclusions. No data, text, or theories by others are presented as if they were our own. Data we used, the data processing and inference phases do not contain any user personal information. This work does not have the potential to be used for policing or the military.

\item Consent to participate - Not applicable

\item Consent for publication - Not applicable

\item Availability of data and material - Experiments are based on publicly available open-source datasets.

\item Code availability - The code will be made available after the paper is submitted.

\item Authors' contributions - H.X contributed the original idea. X.Z conducted experiments. X.Z and H.X wrote the manuscript. X.Z and H.X share equal technical contribution.
\end{itemize} 

\bibliography{main}


\begin{thebibliography}{35}
\ifx \bisbn   \undefined \def \bisbn  #1{ISBN #1}\fi
\ifx \binits  \undefined \def \binits#1{#1}\fi
\ifx \bauthor  \undefined \def \bauthor#1{#1}\fi
\ifx \batitle  \undefined \def \batitle#1{#1}\fi
\ifx \bjtitle  \undefined \def \bjtitle#1{#1}\fi
\ifx \bvolume  \undefined \def \bvolume#1{\textbf{#1}}\fi
\ifx \byear  \undefined \def \byear#1{#1}\fi
\ifx \bissue  \undefined \def \bissue#1{#1}\fi
\ifx \bfpage  \undefined \def \bfpage#1{#1}\fi
\ifx \blpage  \undefined \def \blpage #1{#1}\fi
\ifx \burl  \undefined \def \burl#1{\textsf{#1}}\fi
\ifx \doiurl  \undefined \def \doiurl#1{\url{https://doi.org/#1}}\fi
\ifx \betal  \undefined \def \betal{\textit{et al.}}\fi
\ifx \binstitute  \undefined \def \binstitute#1{#1}\fi
\ifx \binstitutionaled  \undefined \def \binstitutionaled#1{#1}\fi
\ifx \bctitle  \undefined \def \bctitle#1{#1}\fi
\ifx \beditor  \undefined \def \beditor#1{#1}\fi
\ifx \bpublisher  \undefined \def \bpublisher#1{#1}\fi
\ifx \bbtitle  \undefined \def \bbtitle#1{#1}\fi
\ifx \bedition  \undefined \def \bedition#1{#1}\fi
\ifx \bseriesno  \undefined \def \bseriesno#1{#1}\fi
\ifx \blocation  \undefined \def \blocation#1{#1}\fi
\ifx \bsertitle  \undefined \def \bsertitle#1{#1}\fi
\ifx \bsnm \undefined \def \bsnm#1{#1}\fi
\ifx \bsuffix \undefined \def \bsuffix#1{#1}\fi
\ifx \bparticle \undefined \def \bparticle#1{#1}\fi
\ifx \barticle \undefined \def \barticle#1{#1}\fi
\bibcommenthead
\ifx \bconfdate \undefined \def \bconfdate #1{#1}\fi
\ifx \botherref \undefined \def \botherref #1{#1}\fi
\ifx \url \undefined \def \url#1{\textsf{#1}}\fi
\ifx \bchapter \undefined \def \bchapter#1{#1}\fi
\ifx \bbook \undefined \def \bbook#1{#1}\fi
\ifx \bcomment \undefined \def \bcomment#1{#1}\fi
\ifx \oauthor \undefined \def \oauthor#1{#1}\fi
\ifx \citeauthoryear \undefined \def \citeauthoryear#1{#1}\fi
\ifx \endbibitem  \undefined \def \endbibitem {}\fi
\ifx \bconflocation  \undefined \def \bconflocation#1{#1}\fi
\ifx \arxivurl  \undefined \def \arxivurl#1{\textsf{#1}}\fi
\csname PreBibitemsHook\endcsname

\bibitem[\protect\citeauthoryear{Cand{\`e}s et~al.}{2006}]{candes2006robust}
\begin{barticle}
\bauthor{\bsnm{Cand{\`e}s}, \binits{E.J.}},
\bauthor{\bsnm{Romberg}, \binits{J.}},
\bauthor{\bsnm{Tao}, \binits{T.}}:
\batitle{Robust uncertainty principles: Exact signal reconstruction from highly
  incomplete frequency information}.
\bjtitle{IEEE Transactions on information theory}
\bvolume{52}(\bissue{2}),
\bfpage{489}--\blpage{509}
(\byear{2006})
\end{barticle}
\endbibitem

\bibitem[\protect\citeauthoryear{Cand{\`e}s and
  Wakin}{2008}]{candes2008introduction}
\begin{barticle}
\bauthor{\bsnm{Cand{\`e}s}, \binits{E.J.}},
\bauthor{\bsnm{Wakin}, \binits{M.B.}}:
\batitle{An introduction to compressive sampling}.
\bjtitle{IEEE signal processing magazine}
\bvolume{25}(\bissue{2}),
\bfpage{21}--\blpage{30}
(\byear{2008})
\end{barticle}
\endbibitem

\bibitem[\protect\citeauthoryear{Baron et~al.}{2008}]{baron2008bayesian}
\begin{botherref}
\oauthor{\bsnm{Baron}, \binits{D.}},
\oauthor{\bsnm{Sarvotham}, \binits{S.}},
\oauthor{\bsnm{Baraniuk}, \binits{R.}}:
Bayesian compressive sensing via belief propagation.
IEEE Transactions on Signal Processing
(2008)
\end{botherref}
\endbibitem

\bibitem[\protect\citeauthoryear{Khosravy
  et~al.}{2020}]{khosravy2020compressive}
\begin{bbook}
\bauthor{\bsnm{Khosravy}, \binits{M.}},
\bauthor{\bsnm{Dey}, \binits{N.}},
\bauthor{\bsnm{Duque}, \binits{C.A.}}:
\bbtitle{Compressive Sensing in Healthcare},
pp. \bfpage{1}--\blpage{3}.
\bpublisher{Academic Press}, \blocation{???}
(\byear{2020})
\end{bbook}
\endbibitem

\bibitem[\protect\citeauthoryear{Duarte and Eldar}{2011}]{duarte2011structured}
\begin{barticle}
\bauthor{\bsnm{Duarte}, \binits{M.F.}},
\bauthor{\bsnm{Eldar}, \binits{Y.C.}}:
\batitle{Structured compressed sensing: From theory to applications}.
\bjtitle{IEEE Transactions on signal processing}
\bvolume{59}(\bissue{9}),
\bfpage{4053}--\blpage{4085}
(\byear{2011})
\end{barticle}
\endbibitem

\bibitem[\protect\citeauthoryear{Hormati et~al.}{2009}]{hormati2009estimation}
\begin{botherref}
\oauthor{\bsnm{Hormati}, \binits{A.}},
\oauthor{\bsnm{Karbasi}, \binits{A.}},
\oauthor{\bsnm{Mohajer}, \binits{S.}},
\oauthor{\bsnm{Vetterli}, \binits{M.}}:
An estimation theoretic approach for sparsity pattern recovery in the noisy
  setting.
arXiv preprint
(2009)
\end{botherref}
\endbibitem

\bibitem[\protect\citeauthoryear{Kerkouche
  et~al.}{2020}]{kerkouche2020compression}
\begin{bchapter}
\bauthor{\bsnm{Kerkouche}, \binits{R.}},
\bauthor{\bsnm{Ács}, \binits{G.}},
\bauthor{\bsnm{Castelluccia}, \binits{C.}},
\bauthor{\bsnm{Genev{\`e}s}, \binits{P.}}:
\bctitle{Compression boosts differentially private federated learning}.
In: \bbtitle{European Symposium on Security and Privacy}
(\byear{2020})
\end{bchapter}
\endbibitem

\bibitem[\protect\citeauthoryear{Tardivel and Bogdan}{2022}]{tardivel2022sign}
\begin{botherref}
\oauthor{\bsnm{Tardivel}, \binits{P.J.}},
\oauthor{\bsnm{Bogdan}, \binits{M.}}:
On the sign recovery by lasso, thresholded lasso and thresholded basis pursuit
  denoising.
Scandinavian Journal of Statistics
(2022)
\end{botherref}
\endbibitem

\bibitem[\protect\citeauthoryear{Wimalajeewa and
  Varshney}{2016}]{wimalajeewa2016sparse}
\begin{botherref}
\oauthor{\bsnm{Wimalajeewa}, \binits{T.}},
\oauthor{\bsnm{Varshney}, \binits{P.}}:
Sparse signal detection with compressive measurements via partial support set
  estimation.
IEEE Transactions on Signal and Information Processing over Networks
(2016)
\end{botherref}
\endbibitem

\bibitem[\protect\citeauthoryear{Niu}{2022}]{niu2022robust}
\begin{botherref}
\oauthor{\bsnm{Niu}, \binits{J.}}:
Robust deep compressive sensing with recurrent-residual structural constraints.
IEEE Transactions on Computational Imaging
(2022)
\end{botherref}
\endbibitem

\bibitem[\protect\citeauthoryear{Barber and
  Cand{\`e}s}{2016}]{barber2016knockoff}
\begin{botherref}
\oauthor{\bsnm{Barber}, \binits{R.}},
\oauthor{\bsnm{Cand{\`e}s}, \binits{E.}}:
A knockoff filter for high-dimensional selective inference.
Annals of Statistics
(2016)
\end{botherref}
\endbibitem

\bibitem[\protect\citeauthoryear{Dai and Barber}{2016}]{dai2016knockoff}
\begin{bchapter}
\bauthor{\bsnm{Dai}, \binits{R.}},
\bauthor{\bsnm{Barber}, \binits{R.}}:
\bctitle{The knockoff filter for fdr control in group-sparse and multitask
  regression}.
In: \bbtitle{International Conference on Machine Learning}
(\byear{2016})
\end{bchapter}
\endbibitem

\bibitem[\protect\citeauthoryear{Ren and Barber}{2022}]{ren2022derandomized}
\begin{botherref}
\oauthor{\bsnm{Ren}, \binits{Z.}},
\oauthor{\bsnm{Barber}, \binits{R.}}:
Derandomized knockoffs: leveraging e-values for false discovery rate control
(2022)
\end{botherref}
\endbibitem

\bibitem[\protect\citeauthoryear{Cao et~al.}{2021}]{cao2021controlling}
\begin{botherref}
\oauthor{\bsnm{Cao}, \binits{Y.}},
\oauthor{\bsnm{Sun}, \binits{X.}},
\oauthor{\bsnm{Yao}, \binits{Y.}}:
Controlling the false discovery rate in transformational sparsity: Split
  knockoffs.
Journal of the Royal Statistical Society Series B: Statistical Methodology
(2021)
\end{botherref}
\endbibitem

\bibitem[\protect\citeauthoryear{Machkour
  et~al.}{2021}]{machkour2021terminating}
\begin{botherref}
\oauthor{\bsnm{Machkour}, \binits{J.}},
\oauthor{\bsnm{Muma}, \binits{M.}},
\oauthor{\bsnm{Palomar}, \binits{D.}}:
The terminating-random experiments selector: Fast high-dimensional variable
  selection with false discovery rate control
(2021)
\end{botherref}
\endbibitem

\bibitem[\protect\citeauthoryear{Cand{\`e}s and Su}{2015}]{candes2015slope}
\begin{botherref}
\oauthor{\bsnm{Cand{\`e}s}, \binits{E.}},
\oauthor{\bsnm{Su}, \binits{W.J.}}:
Slope is adaptive to unknown sparsity and asymptotically minimax.
arXiv preprint
(2015)
\end{botherref}
\endbibitem

\bibitem[\protect\citeauthoryear{Machkour et~al.}{2024}]{machkour2024sparse}
\begin{bchapter}
\bauthor{\bsnm{Machkour}, \binits{J.}},
\bauthor{\bsnm{Breloy}, \binits{A.}},
\bauthor{\bsnm{Muma}, \binits{M.}},
\bauthor{\bsnm{Palomar}, \binits{D.}},
\bauthor{\bsnm{Pascal}, \binits{F.}}:
\bctitle{Sparse pca with false discovery rate controlled variable selection}.
In: \bbtitle{IEEE International Conference on Acoustics, Speech, and Signal
  Processing}
(\byear{2024})
\end{bchapter}
\endbibitem

\bibitem[\protect\citeauthoryear{Tibshirani}{1996}]{tibshirani1996regression}
\begin{barticle}
\bauthor{\bsnm{Tibshirani}, \binits{R.}}:
\batitle{Regression shrinkage and selection via the lasso}.
\bjtitle{Journal of the Royal Statistical Society Series B: Statistical
  Methodology}
\bvolume{58}(\bissue{1}),
\bfpage{267}--\blpage{288}
(\byear{1996})
\end{barticle}
\endbibitem

\bibitem[\protect\citeauthoryear{Xing et~al.}{2019}]{xing2019controlling}
\begin{botherref}
\oauthor{\bsnm{Xing}, \binits{X.}},
\oauthor{\bsnm{Zhao}, \binits{Z.}},
\oauthor{\bsnm{Liu}, \binits{J.S.}}:
Controlling false discovery rate using gaussian mirrors.
Journal of the American Statistical Association
(2019)
\end{botherref}
\endbibitem

\bibitem[\protect\citeauthoryear{Abramovich
  et~al.}{2005}]{abramovich2005adapting}
\begin{botherref}
\oauthor{\bsnm{Abramovich}, \binits{F.}},
\oauthor{\bsnm{Benjamini}, \binits{Y.}},
\oauthor{\bsnm{Donoho}, \binits{D.}},
\oauthor{\bsnm{Johnstone}, \binits{I.}}:
Adapting to unknown sparsity by controlling the false discovery rate
(2005)
\end{botherref}
\endbibitem

\bibitem[\protect\citeauthoryear{Emery and Keich}{2019}]{emery2019controlling}
\begin{botherref}
\oauthor{\bsnm{Emery}, \binits{K.}},
\oauthor{\bsnm{Keich}, \binits{U.}}:
Controlling the fdr in variable selection via multiple knockoffs
(2019)
\end{botherref}
\endbibitem

\bibitem[\protect\citeauthoryear{Barber and Cand{\`e}s}{2015}]{Barber2015}
\begin{barticle}
\bauthor{\bsnm{Barber}, \binits{R.F.}},
\bauthor{\bsnm{Cand{\`e}s}, \binits{E.J.}}:
\batitle{Controlling the false discovery rate via knockoffs}.
\bjtitle{The Annals of Statistics}
\bvolume{43}(\bissue{5}),
\bfpage{2055}--\blpage{2085}
(\byear{2015})
\doiurl{10.1214/15-AOS1337}
\end{barticle}
\endbibitem

\bibitem[\protect\citeauthoryear{Cand{\`e}s et~al.}{2018}]{candes2018panning}
\begin{barticle}
\bauthor{\bsnm{Cand{\`e}s}, \binits{E.}},
\bauthor{\bsnm{Fan}, \binits{Y.}},
\bauthor{\bsnm{Janson}, \binits{L.}},
\bauthor{\bsnm{Lv}, \binits{J.}}:
\batitle{Panning for gold: {M}odel-{X} knockoffs for high-dimensional
  controlled variable selection}.
\bjtitle{Journal of the Royal Statistical Society: Series B (Statistical
  Methodology)}
\bvolume{80}(\bissue{3}),
\bfpage{551}--\blpage{577}
(\byear{2018})
\doiurl{10.1111/rssb.12265}
\end{barticle}
\endbibitem

\bibitem[\protect\citeauthoryear{Benjamini and
  Hochberg}{1995}]{benjamini1995controlling}
\begin{barticle}
\bauthor{\bsnm{Benjamini}, \binits{Y.}},
\bauthor{\bsnm{Hochberg}, \binits{Y.}}:
\batitle{Controlling the false discovery rate: a practical and powerful
  approach to multiple testing}.
\bjtitle{Journal of the Royal statistical society: series B (Methodological)}
\bvolume{57}(\bissue{1}),
\bfpage{289}--\blpage{300}
(\byear{1995})
\end{barticle}
\endbibitem

\bibitem[\protect\citeauthoryear{McDonald}{2009}]{mcdonald2009ridge}
\begin{barticle}
\bauthor{\bsnm{McDonald}, \binits{G.C.}}:
\batitle{Ridge regression}.
\bjtitle{Wiley Interdisciplinary Reviews: Computational Statistics}
\bvolume{1}(\bissue{1}),
\bfpage{93}--\blpage{100}
(\byear{2009})
\end{barticle}
\endbibitem

\bibitem[\protect\citeauthoryear{Wang et~al.}{2023}]{facial}
\begin{botherref}
\oauthor{\bsnm{Wang}, \binits{Q.}},
\oauthor{\bsnm{Zhou}, \binits{Y.}},
\oauthor{\bsnm{Ghassemi}, \binits{P.}},
\oauthor{\bsnm{Chenna}, \binits{D.}},
\oauthor{\bsnm{Chen}, \binits{M.}},
\oauthor{\bsnm{Casamento}, \binits{J.}},
\oauthor{\bsnm{Pfefer}, \binits{J.}},
\oauthor{\bsnm{McBride}, \binits{D.}}:
Facial and oral temperature data from a large set of human subject volunteers
  (version 1.0.0).
PhysioNet
(2023)
\end{botherref}
\endbibitem

\bibitem[\protect\citeauthoryear{of~the California Institute~of
  Technology}{2024}]{mars}
\begin{botherref}
\oauthor{\bsnm{Technology}, \binits{J.P.L.}}:
Mars Asteroid Observation Dataset.
\url{https://ssd.jpl.nasa.gov/sbdb_query.cgi}.
Accessed: 2025-05-10
(2024)
\end{botherref}
\endbibitem

\bibitem[\protect\citeauthoryear{Bertin-Mahieux
  et~al.}{2011}]{million_song_2011}
\begin{bchapter}
\bauthor{\bsnm{Bertin-Mahieux}, \binits{T.}},
\bauthor{\bsnm{Ellis}, \binits{D.P.W.}},
\bauthor{\bsnm{Whitman}, \binits{B.}},
\bauthor{\bsnm{Lamere}, \binits{P.}}:
\bctitle{The million song dataset}.
In: \bbtitle{{Proceedings of the 12th International Conference on Music
  Information Retrieval ({ISMIR} 2011)}}
(\byear{2011})
\end{bchapter}
\endbibitem

\bibitem[\protect\citeauthoryear{Mathur et~al.}{2025}]{nanofluid}
\begin{barticle}
\bauthor{\bsnm{Mathur}, \binits{P.}},
\bauthor{\bsnm{Shaikh}, \binits{H.}},
\bauthor{\bsnm{Sheth}, \binits{F.}},
\bauthor{\bsnm{Kumar}, \binits{D.}},
\bauthor{\bsnm{Gupta}, \binits{A.K.}}:
\batitle{A computational intelligence framework integrating data augmentation
  and meta-heuristic optimization algorithms for enhanced hybrid nanofluid
  density prediction through machine and deep learning paradigms}.
\bjtitle{IEEE Access}
\bvolume{13},
\bfpage{35750}--\blpage{35779}
(\byear{2025})
\doiurl{10.1109/ACCESS.2025.3543475}
\end{barticle}
\endbibitem

\bibitem[\protect\citeauthoryear{Whiteson}{2014}]{higgs_dataset}
\begin{botherref}
\oauthor{\bsnm{Whiteson}, \binits{D.}}:
{HIGGS}.
UCI Machine Learning Repository.
{DOI}: https://doi.org/10.24432/C5V312
(2014)
\end{botherref}
\endbibitem

\bibitem[\protect\citeauthoryear{Archive}{2025}]{kepler}
\begin{botherref}
\oauthor{\bsnm{Archive}, \binits{N.E.}}:
Kepler Exoplanet Dataset.
\url{https://api.nasa.gov/}.
Accessed: 2025-05-10
(2025)
\end{botherref}
\endbibitem

\bibitem[\protect\citeauthoryear{Bock}{2004}]{telescope}
\begin{botherref}
\oauthor{\bsnm{Bock}, \binits{R.}}:
{MAGIC Gamma Telescope}.
UCI Machine Learning Repository.
{DOI}: https://doi.org/10.24432/C52C8B
(2004)
\end{botherref}
\endbibitem

\bibitem[\protect\citeauthoryear{Chapelle and Chang}{2011}]{yahoo}
\begin{bchapter}
\bauthor{\bsnm{Chapelle}, \binits{O.}},
\bauthor{\bsnm{Chang}, \binits{Y.}}:
\bctitle{Yahoo! learning to rank challenge overview}.
In: \beditor{\bsnm{Chapelle}, \binits{O.}},
\beditor{\bsnm{Chang}, \binits{Y.}},
\beditor{\bsnm{Liu}, \binits{T.-Y.}} (eds.)
\bbtitle{Proceedings of the Learning to Rank Challenge}.
\bsertitle{Proceedings of Machine Learning Research},
vol. \bseriesno{14},
pp. \bfpage{1}--\blpage{24}.
\bconflocation{Haifa, Israel}
(\byear{2011})
\end{bchapter}
\endbibitem

\bibitem[\protect\citeauthoryear{Cai et~al.}{2011}]{CLIME}
\begin{barticle}
\bauthor{\bsnm{Cai}, \binits{T.}},
\bauthor{\bsnm{Liu}, \binits{W.}},
\bauthor{\bsnm{and}, \binits{X.L.}}:
\batitle{A constrained $\ell_1$ minimization approach to sparse precision
  matrix estimation}.
\bjtitle{Journal of the American Statistical Association}
\bvolume{106}(\bissue{494}),
\bfpage{594}--\blpage{607}
(\byear{2011})
\end{barticle}
\endbibitem

\bibitem[\protect\citeauthoryear{CANDES and TAO}{2007}]{candes2007dantzig}
\begin{barticle}
\bauthor{\bsnm{CANDES}, \binits{E.}},
\bauthor{\bsnm{TAO}, \binits{T.}}:
\batitle{The dantzig selector: Statistical estimation when p is much larger
  than n}.
\bjtitle{The Annals of Statistics}
\bvolume{35}(\bissue{6}),
\bfpage{2313}--\blpage{2351}
(\byear{2007})
\end{barticle}
\endbibitem

\end{thebibliography}
\end{document}